\documentclass[11pt]{article}
\usepackage{microtype}
\usepackage{floatflt}
\usepackage{graphicx}
\usepackage{booktabs}
\usepackage{bm}
\usepackage{diagbox}
\usepackage{algorithm,algorithmic}
\usepackage{thm-restate}
\usepackage[algo2e,ruled,vlined]{algorithm2e}

\usepackage{amssymb,amsmath}
\usepackage{tikz}
\usepackage{setspace}
\usepackage[pdftex,bookmarksnumbered,bookmarksopen,
colorlinks,citecolor=blue,linkcolor=blue,urlcolor=blue]{hyperref}
\usepackage{framed}
\usepackage{xcolor}
\usepackage{soul}
\usepackage{longtable}

\usepackage{times}
\usepackage{enumitem}
\usepackage{varwidth}
\usepackage{wrapfig}
\usepackage[font=scriptsize]{caption}

\usepackage{bbm}
\usepackage{url}

\usepackage[utf8]{inputenc}
\usepackage[T1]{fontenc}
\usepackage{booktabs}
\usepackage{amsfonts}
\usepackage{nicefrac}

\usepackage{appendix}
\usepackage{amsmath,amsfonts,amsthm}
\usepackage{xspace}
\usepackage{color}
\usepackage{mathrsfs}

\usepackage{booktabs}
\usepackage{comment}

\usepackage{multirow}

\newcommand{\CR}{\text{RR}}
\newcommand{\ECR}{\widehat{\text{RR}}}

\newcommand{\err}{\mathrm{err}}

\newcommand{\Appendix}[1]{the full version for}

\newcommand{\reg}{\color{black}}

\usepackage{natbib}
\setcitestyle{authoryear,open={(},close={)}}

\newtheorem{theorem}{Theorem}[section]
\newtheorem{lemma}[theorem]{Lemma}

\newtheorem{corollary}[theorem]{Corollary}
\newtheorem{remark}{Remark}

\newtheorem{definition}{Definition}

\newcommand{\E}{\mathbf{E}}

\newcommand{\R}{\mathbb{R}}

\renewcommand{\comment}[1]{}

\newcommand{\cA}{\mathcal{A}}
\newcommand{\cB}{\mathcal{B}}

\newcommand{\cD}{\mathcal{D}}
\newcommand{\cF}{\mathcal{F}}

\newcommand{\cH}{\mathcal{H}}

\newcommand{\cL}{\mathcal{L}}

\newcommand{\cY}{\mathcal{Y}}
\newcommand{\cX}{\mathcal{X}}

\newcommand{\bbE}{\mathbb{E}}

\newcommand{\Agree}{\mathsf{Agree}}

\definecolor{colorY}{rgb}{0.7 , 0.7 , 0.2}

\newcommand{\changed}[1]{#1}

\DeclareMathOperator*{\argmin}{argmin}

\title{Robustly-reliable learners under poisoning attacks}

\usepackage{times}
\usepackage{fullpage}

\author{\begin{tabular}{cc}{Maria-Florina Balcan} & {Avrim Blum}\\{Carnegie Mellon University} & {Toyota Technological Institute at Chicago}\\{\tt ninamf@cs.cmu.edu} & {\tt avrim@ttic.edu} \\ & \\ { Steve Hanneke} &   {Dravyansh Sharma}\\{Purdue University} & {Carnegie Mellon University}\\
{\tt steve.hanneke@gmail.com} & {\tt dravyans@cs.cmu.edu}\end{tabular}}

\begin{document}

 \maketitle

\begin{abstract}%
{Data poisoning} attacks, in which an adversary corrupts a training set with the goal of inducing specific desired mistakes, have raised substantial concern: even just the possibility of such an attack can make a user no longer trust the results of a learning system. In this work, we show how to achieve strong robustness guarantees in the face of such attacks across multiple axes.

We provide {\em robustly-reliable predictions}, in which the predicted label is guaranteed to be correct so long as the adversary has not exceeded a given corruption budget, even in the presence of {\em instance targeted attacks}, where the adversary knows the test example in advance and aims to cause a specific failure on that example.
 Our guarantees are substantially stronger than those in prior approaches, which were only able to provide certificates that the prediction of the learning algorithm does not change,  as opposed to certifying that the prediction is correct, as we are able to achieve in our work.
Remarkably, we provide a complete characterization of learnability in this setting, in particular,  nearly-tight matching upper and lower bounds on the region that can be certified, as well as efficient algorithms for computing this region given an ERM oracle. Moreover, for the case of linear separators over logconcave distributions, we provide efficient truly polynomial time algorithms (i.e., non-oracle algorithms) for such robustly-reliable predictions.

We also extend these results to the active setting where the algorithm adaptively asks for labels of specific informative examples, and the difficulty is that the adversary might even be adaptive to this interaction, as well as to the  agnostic learning setting where there is no perfect classifier even over the uncorrupted data.
\end{abstract}

\section{Introduction}
{\bf Overview:}
There has been significant interest in machine learning in recent years in building robust learning systems that are resilient to adversarial attacks, either test time attacks~\citep{goodfellow2014explaining,carlini2017towards,madry2018towards,zhang2019theoretically} or training time attacks \citep{steinhardt2017certified,shafahi2018poison} also known as {\em data poisoning}.
A lot of the effort on providing provable guarantees for such systems has focused on test-time attacks~\citep[e.g.,][]{DBLP:conf/icml/YinRB19,DBLP:conf/alt/AttiasKM19,montasser2019vc,DBLP:conf/nips/MontasserHS20,montasser2021adversarially,goldwasser2020beyond}, and the question of understanding what is fundamentally possible or not under training time attacks or data poisoning is wide open \citep{barreno2006can,levine2020deep}.

In {data poisoning} attacks,  an adversary corrupts a training set used to train a learning algorithm in order to induce some desired behavior. Even just the possibility of such an attack can make a user no longer trust the results of a learning system. Particularly challenging is the prospect of providing formal guarantees for \emph{instance targeted attacks}, where
the adversary has a goal of inducing a mistake on \emph{specific instances}; the difficulty is that the learner does not know which instance the adversary is targeting, so it must try to guard against attacks for essentially every possible instance. In this work, we devolop a general understanding of what robustness guarantees are possible in such cases while simultaneouly analyzing multiple important axes:
\begin{itemize}[leftmargin=*] \itemsep -1pt
\item  {\em Instance targeted attacks}: The adversary can know our test example in advance, applying its full corruption budget with the goal of making the learner fail on this example, and even potentially change its corruptions from one test example to another.
\item {\em Robustly-reliable predictions}: When our algorithm outputs a prediction $y$ with robustness level $\eta$, this is a guarantee that $y$ is correct so long as the target function belongs to the given class and the adversary has corrupted at most an $\eta$ fraction of the training data. For any value $\eta$, we analyze for which points it is possible to provide a prediction with such a strong robustness level; we provide both sample and distribution dependent nearly matching upper and lower bounds on the size of this set, as well as efficient algorithms for constructing it given access to an ERM oracle.
We note that our guarantees are substantially stronger than those in prior work, which were only able to provide certificates of stability (meaning that the prediction of the algorithm does not change), rather than correctness. This is much more desirable because it  implies that our predictions  can truly be trusted.
\item {\em Active learning against adaptive adversaries:} We also address the challenging active learning setting  where the labeled data is expensive to obtain, but the learning algorithm has the power to adaptively ask for labels of select examples from a large pool of unlabaled examples in order to learn accurate classifiers with fewer labeled examples. While this interaction could save labels, it also raises the concern that the adversary could generate a lot of harm by adaptively deciding which points to corrupt based on choices of the algorithm. We provide algorithms that not only learn with fewer labeled examples, but are able to operate against such adaptive adversaries that can make corruptions on the fly.
\item {\em Agnostic learning:} Providing a reasonable extension to the agnostic case where there is no perfect classifier even over the uncorrupted data.
\end{itemize}
While prior work has considered some of these aspects separately, we provide the first results that bring these together, in particular combining reliable predictions with instance-targeting poisoning attacks.
We also provide  nearly-tight upper and lower bounds on guarantees achievable, as well as efficient algorithms.

\smallskip

\noindent {\bf Our Results:}
Instance targeted data poisoning attacks have been of growing concern as learned classifiers are increasingly used in society.
Classic work on data poisoning attacks
(e.g., \cite{valiant1985learning,kearns1993learning,bshouty2002pac,awasthi2017power})
only considers \emph{non}-instance-targeted attacks,
where the goal of the adversary is only to increase
the overall error rate, rather than to cause errors
on specific test points it wishes to target.
Instance targeted poisoning attacks, on the other hand, are more challenging but they are particularly
relevant for modern applications like recommendation engines, fake review detectors and spam filters where training data is likely to include user-generated data, and the results of the learning algorithm could have financial consequences. For example, a company depending on a recommendation engine for advertisement of its product has an incentive to cause a positive classification on the product it produces or to harm a specific competitor.
To defend against such adversaries, we would like to provide predictions with instance-specific correctness guarantees, even when the adversary can use its entire corruption budget to target those instances. 
This has been of significant concern and interest in recent years in machine learning.

%\smallskip

%\noindent
%{\bf Our Results:}

In this work we consider algorithms that provide {\em robustly-reliable} predictions, which are guaranteed to be correct under well-specified assumptions even in the face of targeted attacks.  In particular, given a test input $x$, a robustly-reliable classifier outputs both a prediction $y$ and a robustness level $\eta$, with a guarantee that $y$ is correct unless one of two bad events has occurred: (a) the true target function does not belong to its given hypothesis class $\cH$ or (b) an adversary has corrupted more than an $\eta$ fraction of the training data.  Such a guarantee intrinsically has in it a notion of targeting, because the prediction is guaranteed to be correct even if the adversarial corruptions in (b) were designed specifically to target $x$.
Note that it is possible to produce a trivial (and useless) robustly-reliable classifier that always outputs $\eta < 0$ (call this  ``abstaining'' or an ``unconfident prediction'').   We will want to produce classifiers that, as much as possible, instead output {\em confident} predictions, that is, predictions with large values of $\eta$.
This leads to several kinds of guarantees one might hope for, because while the adversary cannot cause the algorithm to be incorrect  with a high confidence, it could potentially cause the classifier to produce low-confidence predictions in a targeted way.

In this work, we demonstrate an optimal robustly-reliable learner $\cL$, and precisely identify the guarantees that can be acheived in this setting, with nearly matching  upper and lower bounds. Specifically,

\begin{itemize}[leftmargin=*] \itemsep -1pt
\item We prove guarantees on the set of test points for which the learner will provide confident predictions, for any given adversarial corruption of the training data (Theorem \ref{thm:empirical-ub}), or even more strongly, for all bounded instance-targeted corruptions of the training set (Theorem \ref{thm:theta-ub}). Intuitively speaking, we show that the set of points on which our learner should be confident are those that belong to the region of agreement of low error hypotheses.   Our learner shows more confidence for points in the agreement regions of larger radii around the target hypothesis. We also show how $\cL$ may be implemented efficiently given an ERM oracle (Theorem \ref{thm:erm}) by running the oracle on datasets where multiple copies of the test point are added with different possible labels. We further provide empirical estimates on the set of such test points (Theorem \ref{thm:theta-ub-empirical}), which could help determine if the learner is ready to be fielded in a highly adversarial environment.

\item We provide fundamental lower bounds on which test points a robustly-reliable learner could possibly be confident on (Theorems \ref{thm:strongly_robustly-reliable_lower_bound} and \ref{thm:probably_robustly-reliable_lower_bound_distributional}). We do this through characterizing the set of points an adversary could successfully attack for any learner, which roughly speaking is the region of disagreement of low error classifiers. Intuitively speaking, ambiguity about which the original dataset is leads to ambiguity about which low error classifier is the target concept so we cannot predict confidently on any point in the region of disagreement among these low error concepts.
Our upper bound in Theorem \ref{thm:theta-ub} and lower bound in Theorem \ref{thm:strongly_robustly-reliable_lower_bound} exactly match, and imply that our learner $\cL$ is optimal in its confident set. 
%In particular, if some other reliably-robust learner $\cL'$ has the property that for some point $x$ it provides a prediction with robustness level at least $\eta$ for all $\eta$-bounded perturbations of a training set $S$, then $\cL$ does so for point $x$ as well.

\item For learning linear separators under logconcave marginal  distributions, we show a polynomial time robustly-reliable learning algorithm  under  instance-targeted attacks produced by  a malicious adversary \citep{valiant1985learning} (see Theorem \ref{lem:abl}). The key idea is to first use the learner from the seminal work of \cite{awasthi2017power} 
(that was designed for non-targeted attacks) to first find a low error classifier and then to use that and the geoemetry of the underlying data distribution in order to efficiently find a good approximation of the region of agreement of the low error classifiers.
We also show that the robust-reliability region for this algorithm is near-optimal for this problem, even ignoring  computational efficiency issues (Theorem \ref{thm:probably_robustly-reliable_lower_bound_distributional_mal}).

\item We also study {\em active learning}, extending the classic
disagreement-based active learning techniques, to achieve comparable robust reliability guarantees as above, but with reduced number of labeled examples (Theorems~\ref{thm:acl-ub}, \ref{thm:active-nasty-noise}), even when an adversary can choose which points to corrupt in an adaptive manner (Theorem~\ref{thm:active-nasty-noise}).
This is particularly challenging because the adversary can use its entire corruption budget on only those points whose labels are specifically requested by the learner.

\item Finally, we generalize our results to the agnostic case where there is no perfect classifier even over the uncorrupted data (Theorems \ref{thm:empirical-ub-nr}, \ref{thm:theta-ub-nr}, \ref{thm:strongly_robustly-reliable_lower_bound-nr}, and \ref{thm:distrib-nr}). In this case, the adversary could cause a robustly-reliable learner to make a mistake, but only if {\em every} low-error hypothesis in the class would have also made a mistake on that point as well.
\end{itemize}

\subsection*{Related Work}
\noindent {\bf  Non-instance targeted poisoning attacks.} The classic malicious noise model introduced in~\citep{valiant1985learning} and subsequently analyzed in \citep{kearns1993learning,bshouty2002pac,klivans2009learning,awasthi2017power}
provides one approach to modeling poisoning attacks. However the malicious noise model only captures untargeted attacks --- formally the adversary's goal is to maximize the learner's overall error rate. 

\smallskip

\noindent {\bf Instance targeted poisoning attacks.}~Instance-targeted poisoning attacks were first considered by \cite{barreno2006can}.  \cite{shafahi2018poison} and \cite{suciu2018does} showed empirically that such targeted attacks can be powerful even if the adversary adds only {\em correctly-labeled data} to the training set (called ``clean-label attacks''). Targeted poisoning attacks have generated significant interest in recent years due to the damage they can cause to the trustworthiness of a learning system \citep{mozaffari2014systematic,chen2017targeted,geiping2020witches}.

Past theoretical work on defenses against instance-targeted poisoning attacks has generally focused on producing certificates of stability, indicating when an adversary with a limited budget could not have changed the prediction that was made.  For example, \cite{levine2020deep} propose partitioning training data into $k$ portions, training separate classifiers on each portion, and then using the strength of the majority-vote over those classifiers as such a certificate (since any given poisoned point can corrupt at most one portion).
\cite{gao2021learning} formalize a wide variety of different kinds of adversarial poisoning attacks, and analyze the problem of providing certificates of stability against them in both distribution-independent and distribution-specific settings.
In contrast to those results that certify when a budget-limited adversary could not {\em change} the learner's prediction, our focus is on certifying that the prediction made is {\em correct}.
For example, a learner that always outputs the ``all-negative'' classifier regardless of the training data would be a certifying learner in the sense of \cite{gao2021learning}  (for an arbitrarily large attack budget) and its correctness region would be the probability mass of true negative examples.  In contrast, in our model, outputting a prediction $(y,\eta)$ means that $y$ is guaranteed to be a correct prediction so long as the adversary corrupted at most an $\eta$ fraction of the training data and the target belongs to the given class; so, a learner that always outputs $(y,\eta)$ for $y=-1$ and $\eta\geq 0$ would not be robustly-reliable in our model (unless the given class only had the all-negative function).  We are the first to consider such strong correctness guarantees in the presence of adversarial data poisoning. 

Interestingly, for learning linear separators, our results improve over those of \cite{gao2021learning} even for producing certificates of stability, in that our algorithms in run polynomial time and apply to a much broader class of data distributions (any isotropic log-concave distribution and not only uniform over the unit ball).

\cite{blum2021robust} provide a theoretical analysis for the special case of clean-label poisoning attacks \citep{shafahi2018poison,suciu2018does}. They  analyze the probability mass of attackable instances for various algorithms and hypothesis classes.  However, they do not consider any form of certification or reliability guarantees.

\smallskip
\noindent{\bf Reliable useful learners.}~
Our model can be viewed as a broad generalization of the perfect selective classification model of \cite{el2012active} and the reliable-useful learning model of \cite{rivest1988learning}, which only consider the much simpler setting of learning from noiseless data, to the setting of noisy data and adversarial poisoning attacks.

\section{Formal Setup}

\noindent{\bf Setup}.  Let $\cD$ denote a data distribution over $\cX\times\cY$, where $\cX$ is the instance space and $\cY=\{0,1\}$ is the label space. Let $\cH\subset \cY^{\cX}$ be the concept space. We will primarily assume the realizable case, that is, for some $h^*\in\cH$, for any $(x,y)$ in the support $\texttt{supp}(\cD)$ of the data distribution $\cD$, we have $y=h^*(x)$ (we extend to the non-realizable case in Section \ref{sec:non-realizable}). We will use $\cD_\cX$ to denote the marginal distribution over $\cX$ of unlabeled examples. The learner $\cL$ has access to a corruption $S'$ of a sample $S\sim \cD^m$,
and is expected to output a hypothesis $h_{\cL(S')}\in \cY^{\cX}$ ({\it proper} with respect to $\cH$ if $h_{\cL(S')}\in\cH$).  We will use the 0-1 loss, i.e. $\ell(h,(x,y))=\mathbbm{1}[h(x)\ne y]$.  For a fixed (possibly corrupted) sample $S'$, let $\err_{S'}(h)$ denote the average empirical loss for hypothesis $h$, i.e. $\err_{S'}(h)=\frac{1}{|S'|}\sum_{(x,y)\in S'}\ell(h,(x,y))$. Similarly define $\err_\cD(h)=\bbE_{(x,y)\sim \cD}[\ell(h,(x,h^*(x)))]$. For a sample $S$, let  $\cH_\eta(S)=\{h \in \cH\mid \err_{S}(h)\le\eta\}$ be the set of hypotheses in $\cH$ with at most $\eta$ error on $S$. Similarly for any distribution $\cD$, let $\cH_\eta(\cD)=\{h \in \cH\mid \err_{\cD}(h)\le\eta\}$.
We will consider a class of attacks where the adversary can make arbitrary corruptions to up to an $\eta$ fraction of the training sample $S$. If the adversary can also choose which examples to attack, it corresponds to the {\it nasty} attack model of \cite{bshouty2002pac}.
We formalize the adversary below.

\noindent {\bf Adversary}. Let $d(S,S')=1-\frac{|S\cap S'|}{m}\in[0,1]$ denote the normalized Hamming distance between two samples $S,S'$ with $m=|S|=|S'|$. Let $A(S)$ denote the sample corrupted by adversary $A$. For $\eta\in[0,1]$, let $\cA_\eta$ be the set of adversaries with corruption budget $\eta$ and $\cA_\eta(S)=\{S'\mid d(S,S')\le \eta\}$ denotes the possible corrupted training samples under an attack from an adversary in $\cA_\eta$. Intuitively, if the given sample is $S'$, we would like to give guarantees for learning when $S'\in\cA_\eta(S)$ for some (realizable) uncorrupted sample $S$. Also we will use the convention that $\cA_\eta(S)=\{\}$ for $\eta<0$ to allow the learner to sometimes predict without any guarantees (cf. Definition \ref{def:robustly-reliable learner}). Note that the adversary can change both $x$ and $y$ in an example $(x,y)$ it chooses to corrupt, and can arbitrarily select which $\eta$ fraction to corrupt as in \cite{bshouty2002pac}.

We now define the notion of a {\em robustly-reliable} learner in the face of instance-targeted attacks.  This learner, for any given test example $x$, outputs both a prediction $y$ and a robust reliability level $\eta_x$, such that $y$ is guaranteed to be correct so long as $h^* \in \cH$ and the adversary's corruption budget is $\leq \eta_x$.  This learner then only gets credit for predictions guaranteed to at least a desired value $\eta$.

\begin{definition}[Robustly-reliable learner] A learner $\cL$ is {\bf robustly-reliable} for sample $S'$ w.r.t. concept space $\cH$ if, given $S'$, the learner outputs a function $\cL_{S'}:\cX\rightarrow\cY\times\R$ such that for all $x\in\cX$ if $\cL_{S'}(x)=(y,\eta)$ and if $S'\in \cA_\eta(S)$ for some sample $S$ labeled by concept $h^*\in\cH$, then $y=h^*(x)$. Note that if $\eta<0$, then  $\cA_\eta(S)=\{\}$ and the above condition imposes no requirement on the learner's prediction.
If $\cL_{S'}(x)=(y,\eta)$ then let $h_{\cL(S')}(x)=y$.

Given sample $S$ labeled by $h^*$, the {\bf $\eta$-robustly-reliable region} $\CR^\cL(S,h^*,\eta)$ for learner $\cL$ is the set of points $x\in\cX$ for which given any $S'\in\cA_\eta(S)$ we have that $\cL_{S'}(x)=(y,\eta')$ with $\eta'\ge\eta$. More generally, for a class of adversaries $\cA$ with budget $\eta$, $\CR^\cL_\cA(S,h^*,\eta)$ is the set of points $x\in\cX$ for which given any $S'\in\cA(S)$ we have that $\cL_{S'}(x)=(y,\eta')$ with $\eta'\ge\eta$.
We also define the {\bf empirical $\eta$-robustly-reliable region} $\ECR^\cL(S',\eta) = \{x \in \cX : \cL_{S'}(x) = (y,\eta') \mbox{ for some }\eta'\geq \eta\}$.  So, $\CR^\cL_\cA(S,h^*,\eta) = \cap_{S'\in\cA(S)} \ECR^\cL(S',\eta)$.
\label{def:robustly-reliable learner}
\end{definition}

\begin{remark}
The requirement of security even to targeted attacks appears in two places in Definition \ref{def:robustly-reliable learner}.  First, if a robustly-reliable learner outputs $(y,\eta)$ on input $x$, then $y$ must be correct even if an $\eta$ fraction of the training data had been corrupted specifically to target $x$.  Second, for a point $x$ to be in the $\eta$-robustly-reliable region, it must be the case that for any $S'\in\cA_\eta(S)$ (even if $S'$ is a targeted attack on $x$) we have $\cL_{S'}(x)=(y,\eta')$ for some $\eta'\ge\eta$.  So, points in the $\eta$-robustly-reliable region are points an adversary cannot successfully target with a budget of $\eta$ or less.
\end{remark}

Definition \ref{def:robustly-reliable learner} describes a robustly-reliable learner for a particular (corrupted) sample. We now extend this to the notion of a learner being robustly reliable with high probability for an adversarially-corrupted sample drawn from a given distribution.

\begin{definition}
A learner $\cL$ is a {\bf $(1-\gamma)$-probably robustly-reliable learner} for concept space $\cH$
under marginal $\cD_X$
if for any target function $h^*\in\cH$,
with probability at least $1-\gamma$ over the draw of $S\sim\cD^m$ (where $\cD$ is the distribution over examples labeled by $h^*$ with marginal $\cD_X$), for all $S'\in\cA_\eta(S)$, and for all $x\in\cX$, if $\cL_{S'}(x)=(y,\eta')$ for $\eta'\geq \eta$ then $y= h^*(x)$.
If $\cL$ is a $(1-\gamma)$-probably robustly-reliable learner with $\gamma=0$ for all marginal distributions $\cD_X$, then we say
$\cL$ is {\bf strongly robustly-reliable} for $\cH$.
Note that
a strongly robustly-reliable learner is a robustly-reliable learner in the sense of Definition \ref{def:robustly-reliable learner} for every sample $S'$.

The $\eta$-robustly-reliable correctness for learner $\cL$ for sample $S$ labeled by $h^*$ is given by the probability mass of the robustly-reliable region, $\text{RobC}^\cL(\cD,\eta,S)=\Pr_{x\sim\cD_\cX}[x\in\CR^\cL(S,h^*,\eta)]$.

\label{def:probably robustly-reliable learner}
\end{definition}

In our analysis we will also need several key quanntities -- agreement/disagreement regions and the disagreement coefficient ---  that have been previously used in the disagreement-based active learning literature \citep{balcan2006agnostic,hanneke2007bound,balcan2009agnostic}.

\begin{definition}[Agreement/Disagreement Regions and Disagreement Coefficient] For hypotheses $h,h'$ define disagreement over distribution $\cD$ as $d_{\cD}(h,h')=\Pr_{x\sim \cD_\cX}[h(x) \ne h'(x)]$ and over sample $S$ as $d_{S}(h,h')=\frac{1}{|S|}\sum_{(x,y)\in S}\mathbbm{1}[h(x) \ne h'(x)]$.  For a hypothesis $h\in\cH$ and a value $\epsilon\ge 0$, the ball $\cB^{\cH}_{\cD}(h,\epsilon)$ around $h$
of radius $\epsilon$ is defined as the set $\{h' \in \cH : \Pr_{x\sim \cD_\cX}[h(x) \ne h'(x)] \le \epsilon\}$. For a sample $S$, the ball $\cB^{\cH}_{S}(h,\epsilon)$ is defined similarly with probability taken over $x\in S$.

For a set $H \subseteq \cH$ of hypotheses, let $\text{DIS}(H)= \{x \in \cX : \exists h_1, h_2 \in \cH \text{ s.t. } h_1(x) \ne h_2(x)\} $ be the disagreement region and $\Agree(H)=\cX\setminus\text{DIS}(H)$ be the agreement region of the hypotheses in $H$.

The disagreement coefficient, $\theta_\epsilon$,  of $\cH$ with respect to $h^*$ over $\cD$ is given by
\[\theta_\epsilon=\sup_{r> \epsilon}\frac{\Pr_{\cD_\cX}
[\text{DIS}(\cB_\cD^\cH(h^*, r))]}{r}.\]
\vspace*{-0.3in}

\label{def:dis}
\end{definition}

\section{Robustly Reliable Learners for Instance-Targeted Adversaries} \label{sec:nastynoise}

We now provide a general strongly robustly-reliable learner using the notion of agreement regions (Theorem \ref{thm:empirical-ub}) and show how it can be implemented efficiently given access to an ERM oracle for $\cH$ (Theorem \ref{thm:erm}). We then prove that its robustly reliable region is pointwise optimal (contains the robustly reliable regions of all other such learners), for all values of the adversarial budget $\eta$ (Theorems \ref{thm:theta-ub} and  \ref{thm:strongly_robustly-reliable_lower_bound}).  Recall that a strongly robustly-reliable learner $\cL$ is given a possibly-corrupted sample $S'$ and outputs a function $\cL_{S'}$ such that if $\cL_{S'} = (y,\eta)$ and $S' = \cA_\eta(S)$ for some (unknown) uncorrupted sample $S$ labeled by some (unknown) target concept $h^* \in \cH$, then $y=h^*(x)$.

\begin{theorem}\label{thm:empirical-ub}
{Let $\cH_\eta(S')=\{h \in \cH\mid \err_{S'}(h)\le\eta\}$.} For any hypothesis class $\cH$, there exists a strongly robustly-reliable learner $\cL$ that given $S'$ outputs a function $\cL_{S'}$ such that $$\ECR^{\cL}(S',\eta) \supseteq \Agree(\cH_\eta(S')).$$
\end{theorem}
\begin{proof}
Given sample $S'$, the learner $\cL$ outputs the function $\cL_{S'}(x) =(y,\eta)$ where $\eta$ is the largest value such that $x \in \Agree(\cH_{\eta}(S'))$, and $y$ is the common prediction in that agreement region; if  $x \not\in \Agree(\cH_{\eta}(S'))$ for all $\eta\geq 0$, then $\cL_{S'}(x) =(\bot,-1)$. This is a strongly robustly-reliable learner because if $\cL_{S'}(x) =(y,\eta)$ and $S' \in \cA_\eta(S)$, then $h^* \in \cH_\eta(S')$, so $y=h^*(x)$. Also,
by design, all points in $\Agree(\cH_\eta(S'))$ are given a robust reliability level at least $\eta$.
\end{proof}

We now show how the strongly robustly-reliable learner $\cL$ from Theorem \ref{thm:empirical-ub} can be implemented efficiently given access to an ERM oracle for class $\cH$.

\begin{theorem}
Learner $\cL$ from Theorem \ref{thm:empirical-ub} can be implemented efficiently given an ERM oracle for $\cH$.
\label{thm:erm}
\end{theorem}

\begin{proof}
Given sample $S'$ and test point $x$, for each possible label $y \in \cY$ the learner first computes the minimum value $\epsilon_y$ of the empirical error on $S'$ achievable using $h \in \cH$ subject to $h(x)=y$; that is, $\epsilon_y = \min \{\err_{S'}(h) :  h \in \cH, h(x)=y\}$.  This can be computed efficiently using an ERM oracle by simply running the oracle on a training set consisting of $S'$ and $m+1$ copies of the labeled example $(x,y)$ where $m=|S'|$; this will force ERM to produce a hypothesis $h$ such that $h(x)=y$. Now if $\epsilon_0=\epsilon_1$, then this means that  $x \not\in \Agree(\cH_{\eta}(S'))$ for any value $\eta$ for which $\cH_{\eta}(S')$ is nonempty, so the algorithm can output $(\bot,-1)$.  Otherwise, the algorithm can output the label $y = \argmin_{y'} \{\epsilon_{y'}\}$ and $\eta = \max\{\epsilon_0,\epsilon_1\} - 1/m$.  Specifically, by definition of $\epsilon_y$, this is the largest value of $\eta$ such that $x \in \Agree(\cH_{\eta}(S'))$.
\end{proof}

We now analyze the $\eta$-robustly-reliable region for the algorithm above (Theorems \ref{thm:theta-ub} and \ref{thm:theta-ub-empirical}) and prove that it is optimal over all strongly robustly-reliable learners (Theorem \ref{thm:strongly_robustly-reliable_lower_bound}).
Specifically,
Theorem \ref{thm:theta-ub} provides a guarantee on the size of the $\eta$-robustly-reliable region in terms of properties of $h^*$ and $S$, and then Theorem \ref{thm:theta-ub-empirical} gives an empirically-computable bound.  Thus, these indicate when such an algorithm can be fielded with confidence even in the presence of adversaries that control an $\eta$ fraction of the training data and can modify them at will.

\begin{theorem}\label{thm:theta-ub}
For any hypothesis class $\cH$, the strongly robustly-reliable learner $\cL$ from Theorem \ref{thm:empirical-ub} satisfies the property that for all $S$ and for all $\eta\geq 0$,
$$\CR^\cL(S,h^*,\eta) \supseteq \Agree\left(\cB_S^\cH(h^*,2\eta)\right).$$
Moreover, if  $S\sim \cD^m$ then with probability $1-\delta$, $\Agree(\cB_S^\cH(h^*,2\eta)) \supseteq \Agree(\cB_\cD^\cH(h^*,2\eta+\epsilon))$ for $m=O(\frac{1}{\epsilon^2}(d+\ln\frac{1}{\delta}))$.  So, whp, $\CR^\cL(S,h^*,\eta) \supseteq \Agree(\cB_\cD^\cH(h^*,2\eta+\epsilon)).$
Here $d= {\rm VCdim}(\cH)$.

\end{theorem}

\begin{proof}
By Theorem \ref{thm:empirical-ub}, the empirical $\eta$-robustly-reliable region satisfies $\ECR^{\cL}(S',\eta) \supseteq \Agree(\cH_\eta(S'))$.
The set $\CR^\cL(S,h^*,\eta) = \cap_{S'\in\cA_\eta(S)} \ECR^\cL(S',\eta)$
therefore contains
$$\bigcap_{S' \in \cA_\eta(S)} \Agree(\cH_\eta(S')) = \Agree\left(\cB_S^\cH(h^*,2\eta)\right),$$
where the above equality holds because if $h\in \cB_S^\cH(h^*,2\eta)$ then there exists $S' \in \cA_\eta(S)$ such that $h\in \cH_\eta(S')$, and if $h\in \cH_\eta(S')$ for some $S' \in \cA_\eta(S)$ then $h\in \cB_S^\cH(h^*,2\eta)$.
Finally, by uniform convergence (\cite{anthony2009neural} Theorem 4.10), if $S\sim \cD^m$ for $m=O(\frac{1}{\epsilon^2}(d+\ln\frac{1}{\delta}))$ then with probability at least $1-\delta$ we have $d_\cD(h,h^*)\le d_S(h,h^*)+\epsilon$ for all $h \in \cH$.  {Thus,
$\Agree(\cB_S^\cH(h^*,2\eta)) \supseteq \Agree(\cB_\cD^\cH(h^*,2\eta+\epsilon)$, and $\CR^\cL(S,h^*,\eta) \supseteq \Agree(\cB_\cD^\cH(h^*,2\eta+\epsilon)$.}
\end{proof}

\begin{wrapfigure}{r}{0.4\textwidth}
\centering
\vspace*{-0.3in}

\includegraphics[width=0.35\textwidth]{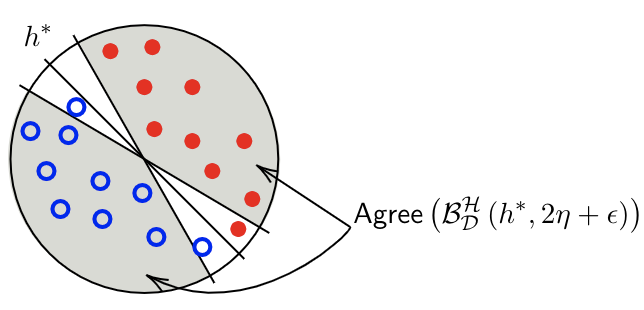}

\vspace*{-0.1in}

\caption{The agreement region of linear separators that
can be given robustly-reliable predictions.}
\vspace*{-0.1in}
\label{figure:agreement region}
\end{wrapfigure}

\begin{remark}[Linear separators for uniform distribution over the unit ball]
The above agreement region can be fairly large for the well-studied setting of learning linear separators under the uniform distribution $\cD$ over the unit ball in $\R^d$, or more generally when the disagreement coefficient (\cite{hanneke2007bound}) is bounded. Since the disagreement coefficient is known to be at most $\pi\sqrt{d}$ in this setting (\cite{hanneke2007bound}), we have that $\Pr[\text{DIS}(\cB_\cD^\cH(h^*,2\eta+\epsilon))]\le \pi\sqrt{d}(2\eta+\epsilon)$ or $\Pr[\Agree(\cB_\cD^\cH(h^*,2\eta+\epsilon))]\ge 1-\pi\sqrt{d}(2\eta+\epsilon)$. For example, if $1\%$ of our data is poisoned and our dataset is large enough so that $\epsilon=0.1\%$, even for $d=9$ we are confident on over $80\%$ of the points. We illustrate the relevant agreement region for linear separators in Figure \ref{figure:agreement region}.
\end{remark}

We can also obtain (slightly weaker) empirical estimates for the size of the robustly-reliable region for the learner $\cL$ from Theorem \ref{thm:empirical-ub}.
Such estimates could help determine when the algorithm can be safely fielded in an adversarial environment.
For example, in federated learning, if an adversary controls processors that together hold an $\eta$ fraction of the training data, and can also cause the learning algorithm to re-run itself at any time, then this would allow one to decide if a low abstention rate can be guaranteed or if additional steps need to be taken (like collecting more training data, or adding more security to processors).

\begin{theorem}\label{thm:theta-ub-empirical}
For any hypothesis class $\cH$, the strongly robustly-reliable learner $\cL$ from Theorem \ref{thm:empirical-ub} satisfies the property that for all $S$, for all $\eta\geq 0$ and for all $S'\in\cA_\eta(S)$,
$$\CR^\cL(S,h^*,\eta) \supseteq \Agree(\cH_{3\eta}(S')).$$
Furthermore, $\cL$ outputs a hypothesis $\hat{h}$ such that
$\CR^\cL(S,h^*,\eta) \supseteq \Agree\left(\cB_{S'}^\cH(\hat{h},4\eta)\right).$
\end{theorem}
\begin{proof} We use the fact that $S'\in\cA_\eta(S)$ to conclude $\Agree(\cH_{3\eta}(S'))\subseteq \Agree\left(\cB_S^\cH(h^*,2\eta)\right)$, which together with Theorem \ref{thm:theta-ub} implies the first claim. Indeed $h\in\cB_S^\cH(h^*,2\eta)$ implies that $\err_{S}(h)\le 2\eta$ and therefore $\err_{S'}(h)\le 3\eta$. Also, if $\err_{S'}(h)\le 3\eta$, then $h\in\cB_{S'}^\cH(\hat{h},4\eta)$ since $\err_{S'}(\hat{h})\le \eta$ again using the triangle inequality, which implies the second claim.
\end{proof}

It turns out that the bound on the robustly-reliable region from Theorem \ref{thm:theta-ub} is essentially optimal. We can show the following lower bound on the ability of any robustly-reliable learner for any hypothesis class $\cH$ to be confident on any point in $\Agree(\cB_S^{\cH}(h^*,2\eta))$. Our upper and lower bounds extend the results of \cite{gao2021learning} for learning halfspaces over the uniform distribution to general hypothesis classes and any distribution with bounded region of disagreement.

\begin{theorem}
Let $\cL$ be a strongly robustly-reliable learner for hypothesis class $\cH$.  Then for any $h^* \in \cH$ and any sample $S$, any point in the $\eta$-robustly-reliable region must lie in the agreement region of $\cB_S^\cH(h^*,2\eta)$. That is,
$$\CR^\cL(S,h^*,\eta) \subseteq \Agree\left(\cB_S^\cH(h^*,2\eta)\right).$$
Moreover, if $S\sim \cD^m$ then with probability $1-\delta$, $\Agree(\cB_S^\cH(h^*,2\eta)) \subseteq \Agree(\cB_\cD^\cH(h^*,2\eta-\epsilon))$ for $m=O(\frac{1}{\epsilon^2}(d+\ln\frac{1}{\delta}))$, where $d = {\rm VCdim}(\cH)$. So, whp, $\CR^\cL(S,h^*,\eta) \subseteq \Agree(\cB_\cD^\cH(h^*,2\eta-\epsilon)).$
\label{thm:strongly_robustly-reliable_lower_bound}
\end{theorem}

\begin{proof}
Let $x \not\in \Agree(\cB_{S}^\cH(h^*,2\eta))$.  We will show that $x$ cannot be in the $\eta$-robustly-reliable region.  First, since $x \not\in \Agree(\cB_{S}^\cH(h^*,2\eta))$, there must exist some $h' \in \cB_{S}^\cH(h^*,2\eta)$ such that $h'(x) \neq h^*(x)$. Next, let $S_X$ be the points in $S$ with labels removed, and let $S_\cA$ denote a labeling of $S_X$ such that exactly half the points in $\Delta_S=\{x\in S_X\mid h'(x)\ne h^*(x)\}$ are labeled according to $h^*$ (the remaining half using $h'$, for convenience assume $|\Delta_S|$ is even). Notice that sample $S=\{(x_i,h^*(x_i)) \mid x_i\in S_X\}$ which labels points in $S_X$ using $h^*$ satisfies $S_\cA\in\cA_\eta(S)$ since $h'\in \cB_{S}(h^*,2\eta)$. Also for $S'=\{(x_i,h'(x_i))\mid x_i\in S_X\}$ labeled by $h'$, we have $S_\cA\in\cA_\eta(S')$. Now, assume for contradiction that $x \in \CR^\cL(S,h^*,\eta)$.  This means that $\cL_{S_\cA}(x)=(y,\eta')$ for some $\eta'\geq \eta$.  However, if $y\ne h^*(x)$, the learner is incorrectly confident for (true) dataset $S$ since $S_\cA\in\cA_\eta(S)$.
Similarly, if $y=h^*(x)$, the learner is incorrectly confident for sample $S'$ since $h'(x)\ne h^*(x)$. Thus, $\cL$ is not a strongly robustly-reliable learner and we have a contradiction.

Finally, {by uniform convergence},
if $S\sim \cD^m$ for $m=O(\frac{1}{\epsilon^2}(d+\ln\frac{1}{\delta}))$ then with probability at least $1-\delta$ we have $d_\cD(h,h^*)\ge d_S(h,h^*)-\epsilon$ for all $h \in \cH$.  This implies that
$\Agree(\cB_S^\cH(h^*,2\eta)) \subseteq \Agree(\cB_\cD^\cH(h^*,2\eta-\epsilon)$
and so $\CR^\cL(S,h^*,\eta) \subseteq \Agree(\cB_\cD^\cH(h^*,2\eta-\epsilon)$ as desired.
\end{proof}

We can also extend the above result to $(1-\gamma)$-probably robustly-reliable learners, as follows.

\begin{theorem}
Let $\cL$ be a $(1-\gamma)$-probably robustly-reliable learner for hypothesis class $\cH$ under marginal $\cD_X$. For any $h^*\in\cH$, given a large enough sample size $m=|S|\ge\frac{c}{\epsilon^2}\ln\frac{1}{\delta}$, we have
$$\bbE_{S\sim\cD^m}[\text{RobC}^\cL(\cD,\eta,S)]\le \Pr[\Agree(\cB_\cD^\cH(h^*,2\eta-\epsilon))]+2\gamma+\delta,$$
where $c$ is an absolute constant and $\cD$ is the distribution with marginal $\cD_X$ consistent with $h^*$.
\label{thm:probably_robustly-reliable_lower_bound_distributional}
\end{theorem}

\begin{proof}
Let  $\cL$ be a $(1-\gamma)$-probably robustly-reliable learner for $\cH$ under marginal $\cD_X$ and let $h^*\in\cH$.  Let $x\not\in \Agree(\cB_{\cD}^\cH(h^*,2\eta-\epsilon))$. To prove the theorem, it suffices to prove that $\Pr_{S \sim \cD^m}[x \in \CR^\cL(S,h^*,\eta) ] \leq 2\gamma+\delta$.

Select some $h' \in \cB_{\cD}^\cH(h^*,2\eta-\epsilon)$ such that $h'(x)\neq h^*(x)$; such an $h'$ exists by definition of the disagreement region.  Let $S \sim \cD^m$ and define $S'=\{(x,h'(x))\mid (x,h^*(x))\in S\}$. Note that $S'\sim\cD'^m$, where $\cD'$ is a data distribution with the same marginal as $\cD$ but consistent with $h'$.  We now consider three bad events of total probability at most $2\gamma+\delta$: (A) $\cL$ is not robustly-reliable for all datasets in $\cA_\eta(S)$, (B) $\cL$ is not robustly-reliable for all datasets in $\cA_\eta(S')$, and (C) $d_S(h',h^*) > 2\eta$.  Indeed events (A) and (B) occur with probability at most $\gamma$ each since $\cL$ is given to be a $(1-\gamma)$-probably robustly-reliable learner. Since $h' \in \cB_{\cD}^\cH(h^*,2\eta-\epsilon)$, for any $x\sim\cD_X$, the probability that $h'(x)\ne h(x)$ is at most $2\eta-\epsilon$. By Hoeffding's inequality, if $m\ge \frac{1}{2\epsilon^2}\ln\frac{1}{\delta}$, we have that $d_S(h',h^*) > 2\eta$ (i.e., event (C)) occurs with probability at most $\delta$.

We claim that if none of these bad events occur, then $x \not\in \CR^\cL(S,h^*,\eta)$.
To prove this, assume for contradiction that none of the bad events occur and $x \in \CR^\cL(S,h^*,\eta)$. Let $\Tilde{S}_{h',h^*}$ denote a relabeling of $S$ such that exactly half the points in $\Delta_S=\{x\mid (x,y)\in S, h'(x)\ne h^*(x)\}$ are labeled according to $h^*$ (the remaining half using $h'$, for convenience assume $|\Delta_S|$ is even). Note that since bad event (C) did not occur, this means that $\Tilde{S}_{h',h^*} \in \cA_\eta(S)$ and $\Tilde{S}_{h',h^*} \in \cA_\eta(S')$,  Now, since $x \in \CR^\cL(S,h^*,\eta)$ and $\Tilde{S}_{h',h^*} \in \cA_\eta(S)$, it must be the case that $\cL_{\Tilde{S}_{h',h^*}}(x) = (y,\eta')$ for some $\eta'\geq \eta$. But, if $y \neq h^*(x)$ then this implies bad event (A), and if $y \neq h'(x)$ then this implies bad event (B).  So, $x$ cannot be in $\CR^\cL(S,h^*,\eta)$ as desired.
\end{proof}

We can strengthen Theorem \ref{thm:probably_robustly-reliable_lower_bound_distributional} for learners that satisfy a somewhat stronger condition that with probability at least $1-\gamma$ over the draw of the {\em unlabeled} sample $S_X$, for all targets $h^* \in \cH$, the sample $S$ produced by labeling $S_X$ by $h^*$ should be a good sample. See Theorem \ref{thm:probably_robustly-reliable_lower_bound_sample} in Appendix \ref{sec:additional}.

\section{Robustly-Reliable Learners for Instance-Targeted Malicious Noise}\label{sec:malicious}
So far, our noise model has allowed the adversary to corrupt an {\em arbitrary} $\eta$ fraction of the training examples.  We now turn to the classic malicious noise model \citep{valiant1985learning,kearns1993learning} in which the points an adversary may corrupt are selected {\em at random}: each point independently with probability $1-\eta$ is drawn from $\cD$ and with probability $\eta$ is chosen adversarially.  Roughly, the malicious adversary model corresponds to a data poisoner that can {\em add} poisoned points to the training set (because the ``clean'' points are a true random sample from $\cD$) whereas the $\cA_\eta$ noise model corresponds to an adversary that can both add {\em and remove} points from the training set.

Traditionally, the malicious noise model has been examined in a non-instance-targeted setting. That is, the concern has been on the overall error-rate of classifiers learned in this model.
Here, we consider instance-targeted malicious noise, and provide efficient robustly-reliable learners building on the seminal work of \cite{awasthi2017power}.
To discuss the set of possible adversarial corruptions in the malicious noise model with respect to particular random draws, for training set $S$ and indicator vector $v \in \{0,1\}^{|S|}$, let $\cA^{mal}(S,v)$ denote the set of all possible $S'$ achievable by replacing the points in $S$ indicated by $v$ with arbitrary points in $\cX \times \cY$.
We will be especially interested in $\cA^{mal}(S,v)$ for $v \sim \mbox{Bernoulli}(\eta)^m$.  Formally,

\begin{definition}
Given sample $S$ and indicator vector $v \in \{0,1\}^{|S|}$, let $\cA^{mal}(S,v)$ denote the collection of corrupted samples $S'$ achievable by replacing the points in $S$ indicated by $v$ with arbitrary points in $\cX \times \cY$.  We use $\cA^{mal}_\eta$ to denote $\cA^{mal}(S,v)$ for $v \sim \mbox{Bernoulli}(\eta)^{|S|}$.
\end{definition}

We now define the notion of probably robustly-reliable learners and their robustly-reliable region for the malicious adversary model.
\begin{definition}
A learner $\cL$ is a {\bf $(1-\gamma)$-probably robustly-reliable learner against $\cA^{mal}_\eta$} for class $\cH$
under marginal $\cD_X$
if for any target function $h^*\in\cH$, with probability at least $1-\gamma$ over the draw of $S\sim\cD^m$ and $v \sim \mbox{Bernoulli}(\eta)^m$, for all $S' \in \cA^{mal}(S,v)$ and all $x\in\cX$, if $\cL_{S'}(x)=(y,\eta')$ for $\eta'\geq \eta$ then $y= h^*(x)$.  The {\bf $\eta$-robustly-reliable region} $\CR_{mal}^{\cL}(S,h^*,\eta,v) = \cap_{S' \in \cA^{mal}(S,v)} \ECR^{\cL}(S',\eta)$, that is,
the set of points $x$ given robust-reliability level $\eta'\geq \eta$ for all $S' \in \cA^{mal}(S,v)$. The $\eta$-robustly-reliable correctness is defined as the probability mass of the $\eta$-robustly-reliable region, $\text{RobC}_{mal}^\cL(\cD,\eta,S,v)=\Pr_{\cD_\cX}[\CR_{mal}^{\cL}(S,h^*,\eta,v)]$.
\label{def:pcl_mal_strong}
\end{definition}

In Appendix \ref{sec:relation} we show how this model relates to the $\cA_\eta$ adversary from Section \ref{sec:nastynoise}, and to other forms of the malicious noise model of \cite{kearns1993learning}.

\subsection{Efficient algorithm for linear separators with malicious noise}

For learning linear separators, the above approaches and prior work \cite{gao2021learning} provide inefficient upper bounds, since we do not generically have an efficient ERM oracle for that class.
In the following we show how we can build on an algorithm of \cite{awasthi2017power} for noise-tolerant learning over log-concave distributions to give strong per-point reliability guarantees that scale well with adversarial budget $\eta$. The algorithm of \cite{awasthi2017power} operates by solving an adaptive series of convex optimization problems, focusing on data within a narrow band around its current classifier and using an inner optimization to perform a soft outlier-removal within that band. We will need to modify the algorithm  to fit our setting, and then build on it to produce per-point reliability guarantees.

\begin{theorem}\label{lem:abl}
Let $\cD_X$ be isotropic log-concave over $\R^d$ and $\cH$ be the class of linear separators.
There is a polynomial-time $(1-\delta)$-probably robustly-reliable learner $\cL$ against $\cA_\eta^{mal}$ for class $\cH$ under $\cD_X$ based on Algorithm 2 of \cite{awasthi2017power}, which uses a sample of size $m=\text{poly}(d, \frac{1}{\eta}
, \log(1/\delta))$.
Furthermore, for any $h^* \in \cH$, with probability at least $1-\delta$ over $S\sim\cD^m$ and $v \sim \mbox{Bernoulli}(\eta)^m$, we have
$\Pr_{\cD_X}[\CR^{\cL}_{mal}(S,h^*,\eta,v)] \geq 1-\Tilde{O}(\sqrt{d}\eta)$.
The $\Tilde{O}$-notation suppresses dependence on logarithmic factors and distribution-specific constants.
\end{theorem}

\begin{proof} We will run a deterministic version of Algorithm 2 of \cite{awasthi2017power} (see Appendix \ref{appendix:ABL}) and let $h = h(S')$ be the halfspace output by this procedure. By Theorem \ref{t:malicious.detailed} (an analog of Theorem 4.1 of \cite{awasthi2017power} for the $\cA_\eta^{mal}$ noise model), there is a constant $C_0$ such that $\text{err}_\cD(h)\le C_0\eta$ with probability at least $1-\delta$ over the draw of $S\sim\cD^m$ and $v \sim \mbox{Bernoulli}(\eta)^m$.\footnote{While \cite{awasthi2017power} describe the adversary as making its malicious choices in a sequential order, the results all hold if the adversary makes its choices after the sample $S$ has been fully drawn; that is, if the adversary selects an arbitrary $S' \in \cA^{mal}(S,v)$. What requires additional care is that the results of \cite{awasthi2017power} imply that for any $S' \in \cA^{mal}(S,v)$ the algorithm succeeds with probability $1-\delta$, whereas we want that with probability $1-\delta$ the algorithm succeeds for any $S' \in \cA^{mal}(S,v)$; in particular, the adversary can make its choice after observing any internal randomness in the algorithm. We address this by making the algorithm deterministic, increasing its label complexity. For more discussion, see Appendix \ref{appendix:ABL}.} In the following, we will assume the occurrence of this $1-\delta$ probability event where the learner outputs a low-error hypothesis for any $S' \in \cA^{mal}(S,v)$. We now show how we extend this result to {\em robustly-reliable} learning.

We will first show we can in principle output a robust-reliability value $\eta$ on all points in $\Agree(\cB^{\cH}_{\cD}(h,C_0\eta))$. (Algorithmically, we will do something slightly different because we do not have an ERM oracle and so cannot efficiently test membership in the agreement region). Indeed by Theorem \ref{t:malicious.detailed}, $h^*$ is guaranteed to be in $\cB^{\cH}_{\cD}(h,C_0\eta)$. By the definition of the disagreement region, if $h(x)\ne h^*(x)$ for some $x$, then $x\in \text{DIS}(\cB^{\cH}_{\cD}(h,C_0\eta))$. Therefore every point in $\Agree(\cB^{\cH}_{\cD}(h,C_0\eta))$ can be confidently classified with robust reliability level $\eta$.

Since we cannot efficiently determine if a point lies in $\Agree(\cB^{\cH}_{\cD}(h,C_0\eta))$, algorithmically we instead do the following.  First, for any halfspace $h'$, let $w_{h'}$ denote the unit-length vector such that $h'(x)=sign(\langle w_{h'}, x\rangle)$.  Next, following the argument in the proof of Theorem 14 of \cite{balcan2013active}, we show that for some constant $C_1$, for all values $\alpha$, we have:
\begin{equation}
\Agree\left(\cB^{\cH}_{\cD}(h,C_0\eta)\right)\supseteq \left\{x : ||x|| < \alpha\sqrt{d}\right\}\cap \left\{x :|\langle w_h, x\rangle|\ge C_1\alpha\eta\sqrt{d}\right\}.
\label{eq:proof1}
\end{equation}
Algorithmically, we will provide robustness level $\eta$ for all points satisfying the right-hand-side above (which we can do efficiently), and the containment above implies this is legal. To complete the proof of the theorem, we must (a) give the proof of containment (\ref{eq:proof1}) and (b) prove that the intersection of the right-hand-sides of  (\ref{eq:proof1}) over all $h=h(S')$ for $S' \in \cA^{mal}(S,v)$ has probability mass
$1-\Tilde{O}(\sqrt{d}\eta)$.

We begin with (a) giving the proof of the containment in formula (\ref{eq:proof1}).
First, by Lemma \ref{lem:ilc-angle}  (due to \cite{balcan2013active}), $\cB^{\cH}_{\cD}(h,C_0\eta)$ consists of hypotheses having angle at most $C_1\eta$ to $h$ for some constant $C_1$. So, if
$||x|| < \alpha\sqrt{d}$, then for any $h' \in \cB^{\cH}_{\cD}(h,C_0\eta)$ we have:
$$
    |\langle w_{h'}, x\rangle - \langle w_h, x\rangle| \; \le \; ||w_{h'}-w_h||\cdot||x|| \; < \; C_1\alpha\eta\sqrt{d}.
$$
Thus, if $x$ also satisfies $|\langle w_h, x\rangle| \ge C_1\alpha\eta\sqrt{d}$, we have $\langle w_h, x\rangle\langle w_{h'}, x\rangle > 0$. This implies that $x \in \Agree\left(\cB^{\cH}_{\cD}(h,C_0\eta)\right)$, completing the proof of containment (\ref{eq:proof1}).

Now we prove (b) that the intersection of the right-hand-sides of  (\ref{eq:proof1}) over all $h=h(S')$ for $S' \in \cA^{mal}(S,v)$ has probability mass
$1-\Tilde{O}(\sqrt{d}\eta)$.  To do this, we first show that for all $h$ such that $\err_\cD(h) \leq C_0\eta$ we have:
\begin{equation}
\{x : ||x|| < \alpha\sqrt{d}\}\cap\{x :|\langle w_h, x\rangle|\ge C_1\alpha\eta\sqrt{d}\}
\supseteq
\{x : ||x|| < \alpha\sqrt{d}\}\cap\{x :|\langle w_{h^*}, x\rangle|\ge 2C_1\alpha\eta\sqrt{d}\}.
\label{eq:proof2}
\end{equation}
This follows from the same argument used to prove (a) above.  Specifically, by Lemma \ref{lem:ilc-angle}, $\cB^{\cH}_{\cD}(h^*,C_0\eta)$ consists of hypotheses having angle at most $C_1\eta$ to $w_{h^*}$ for the same constant $C_1$ as above. So, if
$||x|| < \alpha\sqrt{d}$, then for any $h \in \cB^{\cH}_{\cD}(h^*,C_0\eta)$ we have
$$    |\langle w_h, x\rangle - \langle w_{h^*}, x\rangle| \; \le \; ||w_h-w_{h^*}||\cdot||x|| \; < \; C_1\alpha\eta\sqrt{d}.$$  
This means that if $|\langle w_{h^*}, x\rangle|\ge 2C_1\alpha\eta\sqrt{d}$ then $|\langle w_h, x\rangle|\ge C_1\alpha\eta\sqrt{d}$.

To complete the proof of (b), we now just need to show that the probability mass of the right-hand-side of (\ref{eq:proof2}) is at least $1-\Tilde{O}(\sqrt{d}\eta)$. This follows using the argument in the proof of Theorem 14 in \cite{balcan2013active}.
First we use the fact that for an isotropic log-concave distribution over $\R^d$, we have $\Pr_{\cD_\cX}(||x|| \geq \alpha \sqrt{d}) \leq e^{-\alpha+1}$ (Lemma \ref{lem:isotropic_facts}). This means we can choose $\alpha = \ln(1/(\sqrt{d}\eta))$ and ensure that at most an $O(\sqrt{d}\eta)$ probability mass of points $x$ fail to satisfy the first term in the right-hand-side of (\ref{eq:proof2}). For the second term, using the fact that marginals of isotropic log-concave distributions are also isotropic log-concave (in particular, the 1-dimensional marginal given by taking inner-product with $w_{h^*}$) and that the density of a 1-dimensional isotropic log-concave distribution is at most 1 (Lemma \ref{lem:isotropic_facts}), at most an $O(\alpha\eta\sqrt{d})$ probability mass of points $x$ fail to satisfy the second term in the right-hand-side of (\ref{eq:proof2}).  Putting these together yields the theorem.
\end{proof}

\noindent {\bf Remark.} Note that the above result improves over a related distribution-specific result of \cite{gao2021learning} in three ways. First, we are able to certify that the prediction of our algorithm is  correct rather than only certifying that  the prediction would not change due to the adversary intervention. Second, we are able  to handle much more general distribution --- any isotropic logconcave distribution, as opposed to just the uniform distribution. We essentially achieve the same $\Tilde{O}(\sqrt{d}\eta)$ upper bound on the uncertified region but for a much larger class of distributions. Finally, unlike~\cite{gao2021learning}, our algorithm is  {\em polynomial time}.

\smallskip

\noindent  {\bf Lower bounds.} We also have a near-matching lower bound on robust reliability in the $\cA^{mal}_\eta$ model. This lower bound states that points given robust-reliability level $\eta$ must be in $\Agree(\cB_\cD^\cH(h^*,\frac{\eta}{1-\eta}-\epsilon))$ with a high probability, which differs from our upper bound above by a constant factor in the agreement radius.

\begin{theorem}
Let $\cL$ be a $(1-\gamma)$-probably robustly-reliable learner against $\cA^{mal}_\eta$ for hypothesis class $\cH$ under marginal $\cD_X$. For any $h^*\in\cH$, given a large enough sample size $m=|S|\ge\frac{c}{\epsilon^2}\ln\frac{1}{\delta}$ and any $\eta\le\frac{1}{2}$, we have
$$\bbE_{S\sim\cD^m,v \sim \mbox{Bernoulli}(\eta)^m}\text{RobC}^\cL_{mal}(D,\eta,S,v)]\le \Pr\left[\Agree\left(\cB_\cD^\cH\left(h^*,\frac{\eta}{1-\eta}-\epsilon\right)\right)\right]+2\gamma+2\delta,$$
where $c$ is an absolute constant and $\cD$ is the distribution with marginal $\cD_X$ consistent with $h^*$.
\label{thm:probably_robustly-reliable_lower_bound_distributional_mal}
\end{theorem}

\begin{proof}
Let  $\cL$ be a $(1-\gamma)$-probably robustly-reliable learner for $\cH$ under marginal $\cD_X$ and let $h^*\in\cH$.  Let $x\not\in \Agree(\cB_{\cD}^\cH(h^*,\frac{\eta}{1-\eta}-\epsilon))$. It is sufficient to prove that $\Pr_{S \sim \cD^m,v \sim \mbox{Bernoulli}(\eta)^m}[x \in \CR^\cL(S,h^*,\eta,v) ] \leq 2\gamma+2\delta$.

Select some $h' \in \cB_{\cD}^\cH(h^*,\frac{\eta}{1-\eta}-\epsilon)$ such that $h'(x)\neq h^*(x)$; such an $h'$ exists by definition of the disagreement region.  Let $S \sim \cD^m$ and define $S'=\{(x,h'(x))\mid (x,h^*(x))\in S\}$. Note that $S'\sim\cD'^m$, where $\cD'$ is a data distribution with the same marginal as $\cD$ but consistent with $h'$. Let $\Delta_S=\{(x,y)\in S\mid h'(x)\ne h^*(x)\}$ denote the disagreement of $h^*,h'$ on sample $S$ and  $S^v=\{(x_i,y_i)\in S\mid v_i=1\}$ be the points with indicator $v$ set to $1$ which the malicious adversary gets to corrupt.

We now consider four bad events of total probability at most $2\gamma+2\delta$: (A) $\cL$ is not robustly-reliable for all datasets in $\cA^{mal}(S,v)$, (B) $\cL$ is not robustly-reliable for all datasets in $\cA^{mal}(S',v)$, (C) $d(S,S^v) < \eta - \frac{\epsilon}{3}$, and (D) $|\Delta_S\setminus S^v|>(\eta - \frac{\epsilon}{3})|S|$. Indeed events (A) and (B) occur with probability at most $\gamma$ each since $\cL$ is given to be a $(1-\gamma)$-probably robustly-reliable learner. An application of Hoeffding's inequality implies $(C)$ occurs with probability at most $\delta$ for sufficiently large $m\ge \frac{c}{\epsilon^2}\ln\frac{1}{\delta}$ for some absolute constant $c$ since $v \sim \mbox{Bernoulli}(\eta)^m$. Finally, observe that the event $(x,y)\in\Delta_S\setminus S^v$ for any $(x,y)$ in the sample $S$ is a Bernoulli event with probability of occurrence at most $\left(\frac{\eta}{1-\eta}-\epsilon\right)(1-\eta)=\eta-(1-\eta)\epsilon\le \eta-\frac{\epsilon}{2}$, since the events $(x,y)\in\Delta_S$ and $(x,y)\notin S^v$ are independent and $\eta\le \frac{1}{2}$. Another application of Hoeffding's inequality implies a $\delta$-probability bound on $(D)$.

We claim that if none of these bad events occur, then $x \not\in \CR^\cL_{mal}(S,h^*,\eta,v)$.
Assume for contradiction that none of the bad events occur and $x \in \CR^\cL_{mal}(S,h^*,\eta,v)$. We will now proceed to describe an adversarial corruption $\Tilde{S}^{v}_{h',h^*}$ of $S$ which guarantees at least one of the bad event occurs if $x$ is in the stipulated $\eta$-robustly-reliable region. We select a uniformly random sequence of $|\Delta_S\setminus S^v|$ corruptible points from $S^v$ and replace them $(x,1-y)$ for each $(x,y)\in \Delta_S\setminus S^v$. We substitute the remaining points in $S^v$ by $(x,h^*(x))$ for some fixed $x\in \Agree(\cB_{\cD}^\cH(h^*,\frac{\eta}{1-\eta}-\epsilon))$.
We specify this substitution to ensure that all remaining points in the disagreement region occur in pairs with opposite labels, and we do not reveal the corruptible points. The first claim follows by noting that $|\Delta_S\setminus S^v|\le|S|(\eta - \frac{\epsilon}{3})\le|S^v|$ since bad events (C) and (D) did not occur. Clearly $\Tilde{S}^{v}_{h',h^*} \in \cA^{mal}(S,v)$, the above observation further implies that $\Tilde{S}^{v}_{h',h^*} \in \cA^{mal}(S',v)$. Now, since $x \in \CR^\cL_{mal}(S,h^*,\eta,v)$ and $\Tilde{S}^{v}_{h',h^*} \in \cA^{mal}(S,v)$, it must be the case that $\cL_{\Tilde{S}^{v}_{h',h^*}}(x) = (y,\eta')$ for some $\eta'\geq \eta$. But, if $y \neq h^*(x)$ then this implies bad event (A), and if $y \neq h'(x)$ then this implies bad event (B).  So, $x$ cannot be in $\CR^\cL(S,h^*,\eta,v)$ as desired.
\end{proof}

\section{Active Robustly Reliable Learners}

Instance-targeted attacks may also occur in situations where the labeled data is expensive to obtain and one would typically use active learners to learn accurate classifiers with fewer labeled examples. For example, consider translation to a rare language for which an expert can be expensive to obtain. We might have access to proprietary labeled data, but may be charged per label and would like to minimize the cost of obtaining the labels. It is possible that an attacker with access to the data corrupts it in an instance-targeted way.

We consider a pool-based active learning setting (\cite{settles2009active}), where the learner draws (labeled and unlabeled) examples from a large pool of examples which may have been corrupted by an instance-targeted adversary. That is, a dataset $S$ is drawn according to some true distribution $\cD$ and the instance-targeted adversary may corrupt an $\eta$ fraction of the dataset arbitrarily (change both data points and labels). The learner has access to unlabeled examples from the corrupted dataset $S'$, and is allowed to make {\it label} queries for examples in the dataset.
The goal of the learner is to produce a robustly-reliable predictor for the marginal distribution $\cD_X$, and using as few labels as possible. Our (passive) robustly-reliable learner in Theorem \ref{thm:theta-ub} needs $\Tilde{O}(\frac{d}{\epsilon^2})$ examples; in this section we explore if we can \changed{actively robustly-reliably learn with fewer labeled examples.}

To formalize our results on active learning,
we have the following definition.

\begin{definition}[Active Robustly-Reliable Learner]
A learner $\cL$ is said to be an {\bf active robustly-reliable} learner w.r.t. concept space $\cH$ if for any training sample $S'$, the learner has access to unlabeled examples from $S'$ and the ability to request labels, and must output a function $\cL_{S'}:\cX \rightarrow\cY\times\R$. The {\bf label complexity} of $\cL$ is the number of labels the learner requests before outputting $\cL_{S'}$.

An active robustly-reliable learner $\cL$ is said to be {\bf strongly robustly-reliable} against $\cA_\eta$ for class $\cH$ if for any target function $h^*\in\cH$, for any dataset $S$ consistent with $h^*$, for all $S' \in \cA_\eta(S)$, for all $x\in\cX$, if $\cL_{S'}(x)=(y,\eta')$ for $\eta'\geq \eta$ then $y= h^*(x)$.  \changed{The robustly-reliable correctness and robustly-reliable region, for this value $\eta$, are defined the same as for passive learning above.}

An active robustly-reliable learner $\cL$ is said to be {\bf $(1-\delta)$-strongly robustly-reliable} against $\cA_\eta$ for class $\cH$ if for any target function $h^*\in\cH$, for any dataset $S$ consistent with $h^*$, for all $S' \in \cA_\eta(S)$, with probability at least $1-\delta$ over the internal randomness of the learner, for all $x\in\cX$, if $\cL_{S'}(x)=(y,\eta')$ for $\eta'\geq \eta$ then $y= h^*(x)$.
The empirical $\eta$-robustly-reliable region is defined as for
passive learning above.
 \label{def:acl_nasty}
\end{definition}

We will now show two results indicating that
it is possible for the learner to query the labels intelligently to obtain a robustly-reliable algorithm, first in the stronger sense of strongly robustly-reliable learning,
and second in the weaker sense of $(1-\delta)$-strongly robustly-reliable learning (i.e., with high probability over the randomness of the learner).
In both cases, we are also able to obtain a bound on the number of queries made by the active learner, holding with high probability over the draw of the uncorrupted dataset $S$ (and in the $(1-\delta)$-strongly robustly-reliable case, also the randomness of the learner).

First,
let us consider the case of strongly robustly-reliable
active learning.
Note that this case is quite challenging,
since the adversary effectively \emph{knows} what queries
the learner will make. For instance, the adversary may
choose to set it up so that the learner's first $\eta m$
queries are \emph{all} corrupted, so that the learner
cannot even \emph{find} an uncorrupted point until it has
made at least $\eta m$ queries.
In particular, this intuitively means that nontrivial reductions in label complexity
compared to passive learning are only possible in the case of small $\eta$.

Our result (Theorem~\ref{thm:active-nasty-noise} below)
proposes an active learning method that,
for any given $\eta$,
produces an \emph{optimal} $\eta$-robustly-reliable region
(i.e., matching the upper and lower bounds of
Theorems~\ref{thm:theta-ub} and \ref{thm:strongly_robustly-reliable_lower_bound}),
while making a number of label queries that, for small $\eta$,
is significantly smaller than the number $m$ of examples in $S'$
(the corrupted sample),
when the disagreement coefficient $\theta$ is small: specifically,
it makes a number of queries $\tilde{O}( \theta \eta m + \theta d )$.
The algorithm we propose is actually quite simple:
we process the corrupted
data $S'$ in sequence, and for each new point $x_i$ we query its label $y_i$
iff $x_i$ is in the region of disagreement of the set of all $h \in \cH$
that make at most $\eta m$ mistakes among all previously-queried points.
In the end, after processing all $m$ examples in $S'$,
we use the region of agreement of this same set of classifiers
as the empirical $\eta$-robustly-reliable region, predicting using their
agreed-upon label for any test point $x$ in this region.  The result is stated formally as follows.

\begin{theorem}
\label{thm:active-nasty-noise}
For any hypothesis class $\cH$ and $\eta \geq 0$,
there is an active learner $\cL$
that is strongly robustly-reliable against $\cA_\eta$
which,
for any data set $S$
consistent with some $h^* \in \cH$,
\begin{equation*}
\CR^{\cL}(S,h^*,\eta) \supseteq \Agree(\cB_S^\cH(h^*,2\eta)).
\end{equation*}
Moreover, if $\cD$ is a realizable distribution with some
$h^* \in \cH$ as target concept, then for any $\delta \in (0,1)$,
if $m = \Omega\!\left( \frac{\eta+\epsilon}{\epsilon^2} \left( d \log\frac{1}{\epsilon}+\log\frac{1}{\delta}\right) \right)$
(where $d$ is the VC dimension),
and $S \sim \cD^m$,
with probability at least $1-\delta$,
\begin{equation*}
\CR^{\cL}(S,h^*,\eta) \supseteq \Agree(\cB_\cD^\cH(h^*,2\eta+\epsilon)),
\end{equation*}
and with probability at least $1-\delta$,
the number of label queries made by the algorithm is
at most
\begin{equation*}
O\!\left( \theta \left( \eta m + d \log\frac{m}{d} + \log\frac{1}{\delta} \right)\log\!\left(\min\!\left\{m,\frac{1}{\eta}\right\}\right) \right) = \tilde{O}\!\left( \theta \eta m + \theta d \right),
\end{equation*}
where $\theta = \theta_{4\eta}$ is the disagreement coefficient of $\cH$ with respect to $h^*$ over the distribution $\cD$
(Definition \ref{def:dis}).
\end{theorem}

\begin{proof}[Proof of Theorem~\ref{thm:active-nasty-noise}]
Given any hypothesis class $\cH$ and $\eta \geq 0$,
consider the following active learning algorithm $\cL$.
The algorithm initializes sets $V_0 = \cH$, $Q_0 = \{\}$.
Given a data set $S' = \{(x'_1,y'_1),\ldots,(x'_m,y'_m)\}$,
where the algorithm initially observes only $x'_1,\ldots,x'_m$,
the algorithm, processes the data in sequence:
for each $t \leq m$, the algorithm queries for the label
$y'_t$ iff $x'_t \in \text{DIS}(V_{t-1})$;
if it queries, then $Q_t = Q_{t-1} \cup \{(x'_t,y'_t)\}$,
and otherwise $Q_t = Q_{t-1}$;
in either case, $V_t = \{ h \in \cH : \sum_{(x,y) \in Q_t} \mathbbm{1}[h(x) \neq y] \leq \eta m \}$.
To define the predictions of the learner in the end,
for any $x \in \Agree(V_m)$,
define $\cL_{S'}(x) = (y,\eta)$ for the label $y$
agreed upon for $x$ by every $h \in V_t$;
for any $x \notin \Agree(V_m)$, we can define
$\cL_{S'}(x) = (\bot,-1)$.

Note that if $S' \in \cA_\eta(S)$
for a data set $S = \{(x_1,y_1),\ldots,(x_m,y_m)\}$
consistent with some $h^* \in \cH$,
then (since $Q_t$ is a subsequence of $S'$,
and $h^*$ makes at most $\eta m$ mistakes on $S'$),
we maintain the invariant that $h^* \in V_t$ for all $t$.
In particular, this implies $\cL_{S'}(x)$ is well-defined
since $V_m$ is non-empty.  Moreover, this implies that
for any $x$, any label $y$ agreed-upon by all
$h \in V_m$ is necessarily $h^*(x)$.
Thus, $\cL$ satisfies the definition of a strongly
robustly-reliable learner.
Additionally, for any $h \in V_{m}$,
since we also have $h^* \in V_m$,
the algorithm would have queried for the label $y'_t$
of every $x'_t \in \text{DIS}(\{h,h^*\})$.
Since $h$ remains in $V_{m}$, it must be that $h(x'_t) \neq y'_t$
on at most $\eta m$ of these points.
Since any $h \in \cH$ not in $\cB_S^\cH(h^*,2\eta)$
has strictly greater than $2\eta m$ points in $S$
on which it disagrees with $h^*$ and is incorrect,
and the number of such points in $S'$ can be smaller by
at most $\eta m$,
every $h \in V_{m}$ is necessarily in $\cB_S^\cH(h^*,2\eta)$:
that is, $V_{m} \subseteq \cB_S^\cH(h^*,2\eta)$.
It follows immediately that
$\CR^{\cL}(S,h^*,\eta) \supseteq \Agree(\cB_S^\cH(h^*,2\eta))$.

To address the remaining claims,
let $\cD$ be a realizable distribution with some
$h^* \in \cH$ as target concept,
and fix $\epsilon,\delta \in (0,1)$ and $\eta \geq 0$,
and suppose
$m = \Omega\!\left( \frac{\eta+\epsilon}{\epsilon^2} \left( d \log\frac{1}{\epsilon}+\log\frac{1}{\delta}\right) \right)$
with an appropriate numerical constant factor.
Let $S \sim \cD^m$.
By classic relative uniform convergence guarantees
\cite{vapnik:74}
(see Theorem 4.4 of \cite{vapnik:98}),
with probability at least $1-\delta$,
every $h \in \cH$ satisfies
\begin{equation*}
\left| \err_{S}(h) - \err_{\cD}(h) \right|
\leq O\!\left( \sqrt{ \err_{\cD}(h) \frac{1}{m}\left( d \log \frac{m}{d} + \log\frac{1}{\delta} \right)} + \frac{1}{m}\left( d \log\frac{m}{d}+\log\frac{1}{\delta} \right) \right),
\end{equation*}
so that, for an appropriate numerical constant,
if $m = \Omega\!\left( \frac{\eta+\epsilon}{\epsilon^2} \left( d \log\frac{1}{\epsilon}+\log\frac{1}{\delta}\right) \right)$,
every $h \in \cH$ with $\err_{\cD}(h) > 2 \eta + \epsilon$
has $\err_{S}(h) > 2 \eta$.
Thus, on this event,
$\cB_{S}^{\cH}(h^*,2\eta) \subseteq \cB_{\cD}^{\cH}(h^*,2\eta+\epsilon)$.
In particular, by the above analysis of the algorithm under any fixed $S$,
we have $V_{m} \subseteq \cB_{S}^{\cH}(h^*,2\eta)$.
Together, we have that with probability at least $1-\delta$,
$V_{m} \subseteq \cB_{\cD}^{\cH}(h^*,2\eta+\epsilon)$,
which immediately implies that
$\CR^{\cL}(S,h^*,\eta) \supseteq \Agree(\cB_\cD^\cH(h^*,2\eta+\epsilon))$.

Finally, we turn to bounding the number of label queries.
For every $t \in \{1,\ldots,m\}$, define $S_t = \{(x_1,y_1),\ldots,(x_t,y_t)\}$
and $S'_t = \{(x'_1,y'_1),\ldots,(x'_t,y'_t)\}$.
Note that the number of queries equals
\begin{equation}
\label{eqn:active-queries-bound-1}
\sum_{t=1}^{m} \mathbbm{1}[ x'_t \in \text{DIS}(V_{t-1}) ]
\leq \eta m + \sum_{t=1}^{m} \mathbbm{1}[ x_t \in \text{DIS}(V_{t-1}) ].
\end{equation}
Following a similar argument to the analysis of
$\CR^{\cL}$ above, for any $t \leq m$,
since any $h \in \cH$ not in $\cB_{S_t}^\cH\!\left(h^*,\frac{2 \eta m}{t}\right)$
has strictly greater than $2\eta m$ points in $S_{t}$
on which it disagrees with $h^*$ and is incorrect,
and the number of such points in $S'_{t}$ can be smaller by
at most $\eta m$,
every $h \in V_{t}$ is necessarily in $\cB_{S_t}^\cH\!\left(h^*,\frac{2\eta m}{t}\right)$:
that is,
$V_{t} \subseteq \cB_{S_t}^\cH\!\left(h^*,\frac{2\eta m}{t}\right)$.
Together with \eqref{eqn:active-queries-bound-1},
this implies the total number of queries is at most
\begin{equation*}
\eta m + \sum_{t=1}^{m} \mathbbm{1}\!\left[ x_t \in \text{DIS}\!\left(\cB_{S_{t-1}}^\cH\!\left(h^*,\frac{2\eta m}{t-1}\right)\right) \right].
\end{equation*}
By Bernstein's inequality for martingale difference sequences,
with probability at least $1-\delta/2$, the right hand side
above is at most
\begin{align}
& \eta m + \log_{2}\frac{2}{\delta} + 2e \sum_{t=1}^{m} \text{Pr}\!\left( x_t \in \text{DIS}\!\left(\cB_{S_{t-1}}^\cH\!\left(h^*,\frac{2\eta m}{t-1}\right)\right) \middle| S_{t-1} \right)
\notag \\ & = \eta m + \log_{2}\frac{2}{\delta} + 2e \sum_{t=1}^{m} \cD_{X}\!\left( \text{DIS}\!\left(\cB_{S_{t-1}}^\cH\!\left(h^*,\frac{2\eta m}{t-1}\right)\right) \right).
\label{eqn:active-queries-bound-2}
\end{align}

Let $\alpha_{t} = \frac{c}{t}\left( d \log\frac{t}{d} + \log\frac{1}{\delta} \right)$,
for a numerical constant $c > 0$.
Again by relative uniform convergence guarantees,
together with a union bound (over all values of $t$),
for an appropriate numerical constant $c > 0$,
with probability at least $1-\delta/2$,
for every $t \leq m$ and $h \in \cH$,
\begin{equation*}
\err_{\cD}(h) \leq 2 \err_{S_t}(h) + \alpha_{t}.
\end{equation*}
Thus, on this event, for every $t \leq m$,
$\cB_{S_{t}}^\cH\!\left(h^*,\frac{2\eta m}{t}\right) \subseteq \cB_{\cD}^{\cH}\!\left( h^*, \frac{4 \eta m}{t} + \alpha_{t} \right)$.
Plugging into \eqref{eqn:active-queries-bound-2},
by a union bound, with probability at least $1-\delta$,
the total number of queries is at most
\begin{align*}
& \eta m + \log_{2}\frac{2}{\delta} + 2e \sum_{t=1}^{m} \cD_{X}\!\left( \text{DIS}\!\left(\cB_{\cD}^\cH\!\left(h^*,\frac{4\eta m}{t-1}+\alpha_{t-1}\right)\right) \right)
\\ & \leq (1+8e) \eta m + \log_{2}\frac{2}{\delta} + 2e\sum_{4 \eta m < t \leq m} \cD_{X}\!\left( \text{DIS}\!\left(\cB_{\cD}^\cH\!\left(h^*,\frac{4\eta m}{t-1}+\alpha_{t-1}\right)\right) \right)
\\ & \leq (1+8e) \eta m + \log_{2}\frac{2}{\delta} + 2e\sum_{4 \eta m < t \leq m} \theta_{4\eta} \left( \frac{4\eta m}{t-1}+\alpha_{t-1}\right)
\\ & = O\!\left( \theta_{4\eta} \left( \eta m + d \log\frac{m}{d} + \log\frac{1}{\delta} \right) \log\!\left(\min\!\left\{m,\frac{1}{\eta}\right\}\right) \right),
\end{align*}
where the inequality on the second-to-last line is by the definition of the disagreement coefficient,
and the final expression results from summing the harmonic series:
$\sum_{t=a}^{b} \frac{1}{t} = O\!\left( \log \frac{b}{a} \right)$.
\end{proof}

Next, we consider the weaker case of
$(1-\delta)$-strongly robustly-reliable active learning.
Note that in this case, the adversary only corrupts the data pool, but is not directly involved with the learner's queries. That is, the corruption may depend on the dataset $S$ and the target test instance $x$, but it does not depend on the randomness in the algorithm used by the learner to query labels and output a hypothesis.

In this case, our result (Theorem~\ref{thm:acl-ub} below) provides a method that yields
an empirical robustly-reliable region of
comparable size to the results for passive
learning above (up to constant factors),
using a number of queries that is significantly
smaller than the sample complexity of passive learning,
when $\eta$ and the disagreement coefficient are
small.  However, we note that unlike
Theorem~\ref{thm:active-nasty-noise} above,
in this case the number of queries is
\emph{bounded}, in that it does not depend on
the size $m$ of the data set $S$:
that is, it can be expressed purely as a function
of
$\eta$,$\epsilon$,$\delta$,$d$, and $\theta$.
However, it is also worth noting that the number of queries
in Theorem~\ref{thm:active-nasty-noise} is
actually of the \emph{same} order as in Theorem~\ref{thm:acl-ub} (up to log factors)
in the case that
$m = \Theta\!\left( \frac{\eta+\epsilon}{\epsilon^2} \left( d \log\frac{1}{\epsilon}+\log\frac{1}{\delta}\right) \right)$,
so that for the minimal sample size sufficient for the
result on $\CR^{\cL}(S,h^*,\eta)$ in Theorem~\ref{thm:active-nasty-noise}, we essentially lose
nothing compared to the guarantee achievable by the
weaker $(1-\delta)$-strongly robustly-reliable learner
in Theorem~\ref{thm:acl-ub}.

Our $(1-\delta)$-strongly robustly-reliable active learner in Theorem \ref{thm:acl-ub} will use an agnostic active learning algorithm (in the usual learning setting, without robust-reliability requirements). We will use the algorithm proposed by \cite{dasgupta2007general}. The algorithm partitions points seen so far into sets $T$ (labels explicitly queried) and $U$ (labels not queried). For a new point $x$, ERMs $h_y$ on $T$ consistent with $U\cup(x,y)$ for $y\in\{0,1\}$ are learned and the label for $x$ is requested only if the difference in errors of $h_0,h_1$ is small. The algorithm has the following guarantee.

\begin{theorem}[\cite{dasgupta2007general}]\label{thm:dasgupta_aa}
Let $\cH$ denote the hypothesis space. Let ${\cD}$ be a distribution such that ${h}^*=\argmin_{h\in\cH}\err_\cD(h)$ and $\err_\cD(h^*)=\eta$. If ${\theta}_{\eta+\epsilon}$ is the disagreement coefficient of $\cH$ w.r.t. ${h}^*$ over ${\cD}$, then given ${S}\sim{\cD}^m$ with $m\geq c\left(\frac{\eta+\epsilon}{\epsilon^2}\left( d \log \frac{1}{\epsilon}+\log\frac{1}{\delta}\right) \right)$ unlabeled examples for sufficiently large constant $c$, with probability at least $1-\delta$, the algorithm queries at most $O\!\left( {\theta}_{\eta+\epsilon} \frac{(\eta+\epsilon)^2}{\epsilon^2}\left( d \log\frac{1}{\epsilon}+\log\frac{1}{\delta} \right)\log\frac{1}{\epsilon}  \right)$ labeled examples and returns a hypothesis ${h}\in\cH$ with  $\err_\cD({h})\le \err_\cD({h}^*)+\epsilon$.
\end{theorem}

In Theorem \ref{thm:acl-ub} we define an $(1-\delta)$-strongly robustly-reliable active learner which requires fewer labels than the passive learner.

\begin{theorem}\label{thm:acl-ub}
Let $\cH$ be a hypothesis class, and $S$ be a sample consistent with $h^*\in\cH$. If $\hat{\theta}_{\eta+\epsilon}$ is the disagreement coefficient of $\cH$ with respect to $h^*$ over the uniform distribution $U(S)$ over $S$ (see Definition \ref{def:dis}), there exists a $(1-\delta)$-strongly robustly-reliable active learner $\cL$ based on the agnostic active learning algorithm of \cite{dasgupta2007general}, which for some constant $c$, given any $S'\in\cA_\eta(S)$, for sample size $m\geq c\left(\frac{\eta+\epsilon}{\epsilon^2}\left( d \log \frac{1}{\epsilon}+\log\frac{1}{\delta}\right) \right)$, with probability $1-\delta$ the learner queries $O\!\left( \hat{\theta}_{\eta+\epsilon} \frac{(\eta+\epsilon)^2}{\epsilon^2}\left( d \log\frac{1}{\epsilon}+\log\frac{1}{\delta} \right)\log\frac{1}{\epsilon}  \right)$ labeled examples and returns $\cL_{S'}$ which satisfies, with probability at least $1-\delta$,
\changed{$$\ECR^{\cL}(S',\eta)\supseteq\Agree(\cB_S^\cH(h^*,5\eta+2\epsilon)),$$}
where $d$ is the VC dimension of $\cH$. Moreover, if $\cD$ is a realizable distribution consistent with $h^*\in\cH$ and $\theta_{\eta+\epsilon}$ is the disagreement coefficient of $\cH$ with respect to $h^*$ over $\cD$, if $m= \Omega\left(\frac{1}{\epsilon^2}\left( d +\log\frac{1}{\delta}\right) \right)$, for any  $S\sim\cD^m$, given any $S'\in\cA_\eta(S)$, w.p.  $1-2\delta$, $\cL_{S'}$ satisfies
\changed{$\ECR^{\cL}(S',\eta)\supseteq\Agree(\cB_\cD^\cH(h^*,5\eta+3\epsilon))$,}
and queries $O\!\left( \theta_{\eta+\epsilon} \frac{(\eta+\epsilon)^2}{\epsilon^2}\left( d \log\frac{1}{\epsilon}+\log\frac{1}{\delta} \right)\log\frac{1}{\epsilon}  \right)$ labeled examples w.p. $1-3\delta$.
\end{theorem}

\begin{proof} The learner samples points uniformly randomly from the (corrupted) data pool $S'$ and applies the agnostic active learning algorithm of \cite{dasgupta2007general}, i.e. the agnostic active learning algorithm receives unlabeled points drawn uniformly from the data pool $S'$ and labels are revealed for points requested by the algorithm. Let $\hat{h}$ denote the hypothesis output by the algorithm. The learner provides test point $x$ with the largest $\eta$ such that $x\in\Agree(\cB_{S'}^\cH(\hat{h},2\eta+\epsilon))$ and outputs the common $y$ in the agreement region, and outputs $(\bot,-1)$ if $x\notin\Agree(\cB_{S'}^\cH(\hat{h},2\eta+\epsilon))$ for any $\eta\ge 0$. Note that this computation does not need the knowledge of labels and can be performed over the unlabeled pool which the learner has access to.

By Theorem \ref{thm:dasgupta_aa}, we have an upper bound on the number of labels requested by the agnostic active learner given by $O\!\left( \theta'_{\eta+\epsilon} \frac{(\eta+\epsilon)^2}{\epsilon^2}\left( d \log\frac{1}{\epsilon}+\log\frac{1}{\delta} \right)\log\frac{1}{\epsilon}  \right)$ with probability at least $1-\delta$ over the algorithm's internal randomness, where $\theta'_{\eta+\epsilon}$ is the disagreement coefficient of $\cH$ w.r.t. $h^*$ over $\cD'$, with $\cD'$ being the distribution corresponding to uniform draws from $S'$. We will show that $\theta'_{\eta+\epsilon}=O(\hat{\theta}_{\eta+\epsilon})$, which implies the desired bound. Let $\hat{\cD}$ denote $U(S)$. Observe that for each $h\in\cH$ if $d_{\cD'}(h,h^*)\le r$, then $d_{\hat{\cD}}(h,h^*)=\Pr_{\hat{\cD}}[h(x)\ne h^*(x)]= d_S(h,h^*) \le d_{S'}(h,h^*)+\eta= \Pr_{\cD'}[h(x)\ne h^*(x)]+\eta\le r+\eta$.
Thus, $\cB_{\cD'}^{\cH}(h^*,r)\subseteq \cB_{\hat{\cD}}^{\cH}(h^*,r+\eta)$ and $\text{DIS}(\cB_{\cD'}^{\cH}(h^*,r))\subseteq \text{DIS}(\cB_{\hat{\cD}}^{\cH}(h^*,r+\eta))$.
Now,
\begin{align*}
    \theta'_{\eta+\epsilon}&=\sup_{r>\eta+ \epsilon}\frac{\Pr_{\cD'_X}
[\text{DIS}(\cB_{\cD'}^\cH(h^*, r))]}{r}
\\&\le \sup_{r>\eta+ \epsilon}\frac{\Pr_{\hat{\cD}_X}
[\text{DIS}(\cB_{\cD'}^\cH(h^*, r))]+\eta}{r}
\\&\le \sup_{r>\eta+ \epsilon}\frac{\Pr_{\hat{\cD}_X}
[\text{DIS}(\cB_{\hat{\cD}}^\cH(h^*,r+\eta))]+\eta}{r}
\\&\le \sup_{r>\eta+ \epsilon}\frac{\Pr_{\hat{\cD}_X}
[\text{DIS}(\cB_{\hat{\cD}}^\cH(h^*,2r))]+r}{r}
\\&\le 2\hat{\theta}_{\eta+\epsilon}+1= O(\hat{\theta}_{\eta+\epsilon})
\end{align*}

The agnostic active learning algorithm returns a hypothesis $\hat{h}$ with error at most $\err_{\cD'}(\hat{h})\le\err_{\cD'}(h^*)+\epsilon\le\eta+\epsilon$ on $S'$ with failure probability $\delta$.
Since $\err_{S'}(\hat{h})\le \eta+\epsilon$ (if the failure event does not occur) and $\err_{S'}(h^*)\le \eta$, using the triangle inequality $h^*\in\cB_{S'}^\cH(\hat{h},2\eta+\epsilon)$ and we can provide robust-reliability $\eta$ to $x$ if $x\in\Agree(\cB_{S'}^\cH(\hat{h},2\eta+\epsilon))$. Notice $\cB_{S'}^\cH(\hat{h},2\eta+\epsilon)\subseteq \cB_{S'}^\cH(h^*,4\eta+2\epsilon)\subseteq \cB_{S}^\cH(h^*,5\eta+2\epsilon)$. Therefore the \changed{empirical} $\eta$-robustly-reliable region of the learner contains $\Agree(\cB_S^\cH(h^*,5\eta+2\epsilon))$ with probability $1-\delta$ over the learner's randomness.

To establish the distributional results, we need to relate $\theta_{\eta+\epsilon}$ with $\hat{\theta}_{\eta+\epsilon}$ and the robustly-reliable agreement region over the sample to one over the distribution, both using uniform convergence bounds. By uniform convergence (\cite{anthony2009neural} Theorem 4.10), with probability at least $1-\delta$, for each $h\in\cH$ if $d_{\hat{\cD}}(h,h^*)\le r$, then $d_{\cD}(h,h^*) \le  d_S(h,h^*)+\epsilon = d_{\hat{\cD}}(h,h^*)+\epsilon\le r+\epsilon$.
Thus, on this event, $\cB_{\hat{\cD}}^{\cH}(h^*,r)\subseteq \cB_{\cD}^{\cH}(h^*,r+\epsilon)$ and $\text{DIS}(\cB_{\hat{\cD}}^{\cH}(h^*,r))\subseteq \text{DIS}(\cB_{\cD}^{\cH}(h^*,r+\epsilon))$.
Moreover, note that the family
$\{ \text{DIS}(\cB_{\cD}^{\cH}(h^*,r)) : r \geq 0 \}$ has VC dimension $1$,
and therefore by uniform convergence,
with probability at least $1-\delta$,
we have $\Pr_{\hat{\cD}_{X}}[\text{DIS}(\cB_{\cD}^{\cH}(h^*,r))] \leq \Pr_{\cD_{X}}[\text{DIS}(\cB_{\cD}^{\cH}(h^*,r))]+\epsilon$
for all $r \geq 0$.
Supposing both of these events, we have
\begin{align*}
\hat{\theta}_{\eta+\epsilon}&=\sup_{r>\eta+ \epsilon}\frac{\Pr_{\hat{\cD}_X}
[\text{DIS}(\cB_{\hat{\cD}}^\cH(h^*, r))]}{r}
\\&\le \sup_{r>\eta+ \epsilon}\frac{\Pr_{\hat{\cD}_X}
[\text{DIS}(\cB_{\cD}^\cH(h^*,r+\epsilon))]}{r}
\\&\le \sup_{r>\eta+ \epsilon}\frac{\Pr_{\cD_X}
[\text{DIS}(\cB_{\cD}^\cH(h^*, r+\epsilon))]+\epsilon}{r}
\\&\le \sup_{r>\eta+ \epsilon}\frac{\Pr_{\cD_X}
[\text{DIS}(\cB_{\cD}^\cH(h^*,2r))]+r}{r}
\\&= 1+2\sup_{r>\eta+ \epsilon}\frac{\Pr_{\cD_X}
[\text{DIS}(\cB_{\cD}^\cH(h^*,2r))]}{2r}
\\&\leq 2\theta_{\eta+\epsilon}+1= O(\theta_{\eta+\epsilon}),
\end{align*}

By a union bound, the above $3$ events
occur simultaneously, with probability at
least $1-3\delta$.
This implies the desired label complexity bound. The bound on the trusted region follows from another application of the uniform convergence bound which implies $\cB_{S}^\cH(h^*,5\eta+2\epsilon)\subseteq \cB_{\cD}^\cH(h^*,5\eta+3\epsilon)$ with probability $1-\delta$,
from which the stated result follows
by a union bound.
\end{proof}

\begin{remark}
Our above result is essentially a reduction of finding a $(1-\delta)$-strongly robustly-reliable active learner to a general agnostic active learner. We have used the algorithm from \cite{dasgupta2007general} in Theorem \ref{thm:acl-ub}, but we can substitute any other agnostic active learning algorithm to get the corresponding label complexity guarantees.
\end{remark}

We demonstrate an application of the above theorem to the well-studied setting of learning a linear separator, with the data distributed according to the uniform distribution over the unit ball.

\begin{remark}[Linear separators for uniform distribution over the unit ball]
It was shown by \cite{hanneke2007bound} that
$\theta_{\eta+\epsilon}=O(\sqrt{d})$ for this setting.
Theorem \ref{thm:acl-ub} now implies an exponential gain in the label complexity for $\epsilon\approx \eta$ and a gain by a factor of $\Tilde{O}(\frac{1}{\sqrt{d}\eta})$ for $\epsilon\ll \eta$.
\end{remark}

\section{Robustly reliable agnostic learners}\label{sec:non-realizable}

{
So far we have assumed throughout that the uncorrupted samples $S$ are realizable under our concept class $\cH$.
We will now show that our results can be extended to the non-realizable setting, i.e. $\min_{h\in\cH}\err_S(h)>0$, with weaker but still interesting guarantees. Specifically, our algorithm might now produce an incorrect prediction $(y,\eta)$ with $\eta$ greater than the adversary's power, but only if {\em every} hypothesis in $\cH$ with low error on the uncorrupted $S$ would also be incorrect on that example.

We can define a {\it $\nu$-tolerably robustly-reliable} learner in the non-realizable setting as the learner whose reliable predictions agree with every low error hypothesis (error at most $\nu$) on the uncorrupted sample.

\begin{definition}A learner $\cL$ is {\bf $\nu$-tolerably robustly-reliable} for sample $S'$ w.r.t. concept space $\cH$ if, given $S'$, the learner outputs a function $\cL_{S'}:\cX\rightarrow\cY\times\R$ such that for all $x\in\cX$ if $\cL_{S'}(x)=(y,\eta)$
then for all $h^*\in\cH$ such that $\err_S(h^*)\le \nu$ for some $S$ with $S'\in \cA_\eta(S)$, we have
$y=h^*(x)$.

Given sample $S$ such that there is some $h^*\in\cH$  satisfying $\err_S(h^*)\le \nu$, the {\bf $(\nu,\eta)$-robustly-reliable region} $\CR^\cL(S,\nu,\eta)$ for learner $\cL$ is the set of points $x\in\cX$ for which given any $S'\in\cA_\eta(S)$ we have that $\cL_{S'}(x)=(y,\eta')$ with $\eta'\ge\eta$. More generally, for a class of adversaries $\cA$ with budget $\eta$, $\CR^\cL_\cA(S,\nu,\eta)$ is the set of points $x\in\cX$ for which given any $S'\in\cA(S)$ we have that $\cL_{S'}(x)=(y,\eta')$ with $\eta'\ge\eta$.
We also define the {\bf empirical $(\nu,\eta)$-robustly-reliable region} $\ECR^\cL(S', \nu,\eta) = \{x \in \cX : \cL_{S'}(x) = (y,\eta') \mbox{ for some }\eta'\geq \eta\}$.  So, $\CR^\cL_\cA(S,\nu,\eta) = \cap_{S'\in\cA(S)} \ECR^\cL(S',\nu, \eta)$.%

\label{def:robustly-reliable learner nonrealizable}
\end{definition}

Definition \ref{def:robustly-reliable learner nonrealizable} describes the notion of a robustly-reliable learner for a particular (corrupted) sample. Notice that if there are multiple $h \in \cH$ satisfying $\err_S(h)\leq \nu$ and they disagree on $x$ then it must be the case that the algorithm outputs $\eta<0$. Notice also that setting $\nu=0$ in the above definition yields the usual robustly-reliable learner. Similarly to Definition \ref{def:probably robustly-reliable learner}, we now extend this to robustly-reliable with high probability for an adversarially-corrupted sample drawn from a given distribution.

\begin{definition}
A learner $\cL$ is a {\bf $(1-\gamma)$-probably $\nu$-tolerably  robustly-reliable learner} for concept space $\cH$
under marginal $\cD_X$ (where $\cD$ is the distribution over examples with marginal $\cD_X$)
if with probability at least $1-\gamma$ over the draw of $S\sim\cD^m$, for any concept $h^*\in\cH$ such that $\err_\cD(h^*)\le\nu$,
for all $S'\in\cA_\eta(S)$, and for all $x\in\cX$, if $\cL_{S'}(x)=(y,\eta')$ for $\eta'\geq \eta$ then $y= h^*(x)$.
If $\cL$ is a $(1-\gamma)$-probably $\nu$-tolerably  robustly-reliable learner with $\gamma=0$ for all marginal distributions $\cD_X$, then we say
$\cL$ is {\bf $\nu$-tolerably  strongly robustly-reliable} for $\cH$. Note that
a $\nu$-tolerably strongly robustly-reliable learner is a $\nu$-tolerably robustly-reliable learner in the sense of Definition \ref{def:robustly-reliable learner nonrealizable} for every sample $S'$.
Given distribution $\cD$, the $(\nu,\eta)$-robustly-reliable correctness for learner $\cL$ for sample $S$ is given by the probability mass of the robustly-reliable region, $\text{RobC}^\cL(\cD,\nu,\eta,S)=\Pr_{x\sim\cD_\cX}[x\in\CR^\cL(S,\nu,\eta)]$.

\label{def:probably robustly-reliable learner nonrealizable}
\end{definition}

}

We now provide a general $\nu$-tolerably  strongly robustly-reliable learner using the notion of agreement regions (Theorem \ref{thm:empirical-ub-nr}).  Our results here generalize corresponding results from Section \ref{sec:nastynoise}. We first present our learner and a guarantee on the learner's empirical $(\nu,\eta)$-robustly-reliable region given any (possibly corrupted) dataset.
Our algorithm assumes $\nu$ is given.

\begin{theorem}\label{thm:empirical-ub-nr}
{Let $\cH_{\eta+\nu}(S')=\{h \in \cH\mid \err_{S'}(h)\le\eta+\nu\}$.} For any hypothesis class $\cH$, there exists a $\nu$-tolerably strongly robustly-reliable learner $\cL$ (Definition \ref{def:probably robustly-reliable learner nonrealizable}) that given $S'$ outputs a function $\cL_{S'}$ such that $$\ECR^{\cL}(S',\nu,\eta) \supseteq \Agree(\cH_{\eta+\nu}(S')).$$%
\end{theorem}
\begin{proof}
Given sample $S'$, the learner $\cL$ outputs the function $\cL_{S'}(x) =(y,\eta)$ where $\eta$ is the largest value such that $x \in \Agree(\cH_{\eta+\nu}(S'))$, and $y$ is the common prediction in that agreement region; if  $x \not\in \Agree(\cH_{\eta+\nu}(S'))$ for all $\eta\geq 0$, then $\cL_{S'}(x) =(\bot,-1)$. This is a strongly robustly-reliable learner because if $\cL_{S'}(x) =(y,\eta)$ and $S' \in \cA_\eta(S)$ and $\err_S(h^*)\le\nu$, then $h^* \in \cH_{\eta+\nu}(S')$, so $y=h^*(x)$. Also, notice that by design of the algorithm, all points in $\Agree(\cH_{\eta+\nu}(S'))$ will be given robust-reliability level at least $\eta$. The learner may be implemented using an ERM oracle using knowledge of $\nu$ and the construction in Theorem \ref{thm:erm}.
\end{proof}

We now analyze the $(\nu,\eta)$-robustly-reliable region for the algorithm above which we will prove is pointwise optimal over all $\nu$-tolerably strongly robustly-reliable learners.

\begin{theorem}\label{thm:theta-ub-nr}
For any hypothesis class $\cH$, the $\nu$-tolerably strongly robustly-reliable learner $\cL$ from Theorem \ref{thm:empirical-ub-nr} satisfies the property that for all $S$ and for all $\eta\geq 0$ and for any $h^*$ with $\err_S(h^*)\le\nu$,
$$\CR^\cL(S,\nu,\eta)\supseteq \Agree\left(H_{2\eta+\nu}(S)\right)\supseteq \Agree\left(\cB_S^\cH(h^*,2\eta+2\nu)\right),$$
where $\cH_{\alpha}(T)=\{h \in \cH\mid \err_{T}(h)\le \alpha\}$ for any $\alpha\ge 0$, where $T$ may be a sample or a distribution. {Moreover, if  $S\sim \cD^m$ for $m=O(\frac{1}{\epsilon^2}(d+\ln\frac{1}{\delta}))$ then with probability at least $1-\delta$, $\CR^\cL(S,\nu,\eta) \supseteq \Agree(H_{2\eta+\nu+\epsilon}(\cD)).$
Here $d$ denotes the VC dimension of $\cH$.}

\end{theorem}

\begin{proof} By Theorem \ref{thm:empirical-ub-nr}, the empirical $\eta$-robustly-reliable region $\ECR^{\cL}(S',\nu,\eta) \supseteq \Agree(\cH_{\eta+\nu}(S'))$ for any dataset $S'$.
The $(\nu,\eta)$-robustly-reliable region, which is the set of points given robustness level at least $\eta$ for {\em all} $S' \in \cA_\eta(S)$ given that there is $h^*\in\cH$ with $\err_{S}(h^*)\le\nu$, is therefore at least
$$\bigcap_{S' \in \cA_\eta(S)} \Agree(\cH_{\eta+\nu}(S')) = \Agree\left(H_{2\eta+\nu}(S)\right)$$%
where the above holds because if $h\in \cH_{\eta+\nu}(S')$ for some $S' \in \cA_\eta(S)$ then $h\in\cH_{2\eta+\nu}(S)$, and conversely if $h\in\cH_{2\eta+\nu}(S)$ then $h\in \cH_{\eta+\nu}(S')$ for some $S' \in \cA_\eta(S)$. Further, if $h\in\cH_{2\eta+\nu}(S)$ then
$h\in \cB_S^\cH(h^*,2\eta+2\nu)$ since $\err_S(h^*)\le\nu$. This implies $\Agree\left(H_{2\eta+\nu}(S)\right)\supseteq \Agree\left(\cB_S^\cH(h^*,2\eta+2\nu)\right)$. Finally, by uniform convergence, if $S\sim \cD^m$ for $m=O(\frac{1}{\epsilon^2}(d+\ln\frac{1}{\delta}))$ then with probability at least $1-\delta$ we have $\err_\cD(h)\le \err_S(h)+\epsilon$ for all $h \in \cH$.  This implies that
$\Agree(H_{2\eta+\nu}(S)) \supseteq \Agree(H_{2\eta+\nu+\epsilon}(\cD))$
and so $\CR^\cL(S,\nu,\eta) \supseteq \Agree(H_{2\eta+\nu+\epsilon}(\cD))$.
\end{proof}

We will now show a matching lower bound for Theorem \ref{thm:theta-ub-nr}. In contrast to the realizable case, our matching bounds here are in terms of agreement regions of hypotheses with low error on the sample (more precisely $\Agree\left(H_{2\eta+\nu}(S)\right)$), instead of balls around a fixed hypothesis in $\cH$. In Theorems \ref{thm:theta-ub} and \ref{thm:strongly_robustly-reliable_lower_bound}, the size of the robustly-reliable region is $\Agree\left(\cB_S^\cH(h^*,2\eta)\right)$, which in the realizable case is the same as $\Agree\left(H_{2\eta}(S)\right)$. This also implies we do not lose too much of the robustly-reliable region in the agnostic case.

\begin{theorem}
Let $\cL$ be a $\nu$-tolerably strongly robustly-reliable learner for hypothesis class $\cH$.  Then for any sample $S$, any point in the $(\nu,\eta)$-robustly-reliable region must lie in the agreement region of  $\cH_{2\eta+\nu}(S)=\{h \in \cH\mid \err_{S}(h)\le2\eta+\nu\}$. That is,
$$\CR^\cL(S,\nu,\eta) \subseteq \Agree\left(\cH_{2\eta+\nu}(S)\right).$$

\label{thm:strongly_robustly-reliable_lower_bound-nr}
\end{theorem}

\begin{proof}
Let $x \not\in \Agree(\cH_{2\eta+\nu}(S))$.  We will show that $x$ cannot be in the $(\nu,\eta)$-robustly-reliable region.  First, since $x \not\in \Agree(\cH_{2\eta+\nu}(S))$, there must exist some $h_1,h_2 \in \cH_{2\eta+\nu}(S)$ such that $h_1(x) \neq h_2(x)$. Since we are interested in the $(\nu,\eta)$-robustly-reliable region, suppose we have hypothesis $h^*\in\cH$ such that $\err_S(h^*)\le \nu$. Now either $h_1(x)\ne h^*(x)$ or $h_2(x)\ne h^*(x)$. Assume WLOG that $h_1(x)\ne h^*(x)$. Let $S_e^*=\{(x,y)\in S\mid h^*(x)\ne y)\}$ denote the points in the uncorrupted sample where $h^*$ is incorrect. Further let $S_e=\{(x,y)\in S\mid h_1(x)\ne y\}$ be the points where $h_1$ is incorrect on $S$. Finally let $\Tilde{S}$ be a fixed subset of $S_e\setminus S_e^*$ of size $\min\{2\eta m,|S_e\setminus S_e^*|\}$ and let $S'\subset \Tilde{S}$ with $|S'|=|\Tilde{S}|/2$.

We will now construct two sets $S_1$ and $S_\cA$. $S_1$ will be such that $\err_{S_1}(h_1)\le\nu$ and $S_\cA$ will satisfy $S_\cA\in \cA_\eta(S)$ as well as $S_\cA\in \cA_\eta(S_1)$. For any $U\subseteq T$, let $\textsc{flip}(T,U)=\{(x,1-y)\mid (x,y)\in U\}\cup T\setminus U$ denote a sample $T$ with flipped labels for subset $U$. We define $S_1=\textsc{flip}(S,\Tilde{S})$ and $S_\cA=\textsc{flip}(S,S')$. To show $\err_{S_1}(h_1)\le\nu$, note that we have one of two cases. Either $|S_e\setminus S_e^*|\le 2\eta m$, in which case $\Tilde{S}=S_e\setminus S_e^*$ and $h_1$ will we correct on these in $S_1$. So $h_1$ is incorrect on at most $|S_e^*|$ points on $S_1$ and $\err_{S_1}(h_1)\le\nu$. The other case is $|S_e\setminus S_e^*|> 2\eta$, which implies $\Tilde{S}\ge 2\eta m$. Since $\Tilde{S}\subseteq S_e$, we have that $\err_{S_1}(h_1)\le \err_{S}(h_1)-2\eta\le \nu$. Finally since $|\Tilde{S}|\le 2\eta m$ and $|S'|=|\Tilde{S}|/2$, we have $d(S,S_1)\le 2\eta$ and $d(S,S_\cA)\le \eta$. Also, $S_1$ and $S_\cA$ differ on points corresponding to $\Tilde{S}\setminus S'$ with $|\Tilde{S}\setminus S'|=|\Tilde{S}|/2\le \eta m$. Thus, $S_\cA\in \cA_\eta(S)$ as well as $S_\cA\in \cA_\eta(S_1)$.

Now notice that sample $S$ has $\err_S(h^*)\le\nu$ and $S_\cA\in\cA_\eta(S)$. Also $S$ has $\err_{S_1}(h_1)\le\nu$ and $S_\cA\in\cA_\eta(S_1)$. Now, assume for contradiction that $x \in \CR^\cL(S,\nu,\eta)$.  This means that $\cL_{S_\cA}(x)=(y,\eta')$ for some $\eta'\geq \eta$.  However, if $y\ne h^*(x)$, the learner is incorrectly confident for (true) dataset $S$ since $S_\cA\in\cA_\eta(S)$.
Similarly, if $y=h^*(x)$, the learner is incorrectly confident for sample $S_1$ since $h_1(x)\ne h^*(x)$. Thus, $\cL$ is not a $\nu$-tolerably strongly robustly-reliable learner and we have a contradiction.
\end{proof}

We can show that any $2\nu$-tolerably strongly robustly-reliable learner for the concept space can be used to give a $(1-\gamma)$-probably $\nu$-tolerably  robustly-reliable learner over a distribution $\cD$ where the best hypothesis $h^*$ in $\cH$ has error $\err_\cD(h^*)=\nu$.

\begin{theorem}\label{thm:agnostic-distribution-reduction}
Let $\cH$ be the concept space, and $\cD$ be a distribution such that $\min_{h\in\cH}\err_\cD(h)=\nu^*$, and let $\cL$ be a $2\nu^*$-tolerably robustly-reliable learner for $\cH$ for any sample $S'$. If $m\ge \frac{c}{{\nu^*}^2}(d+\ln\frac{1}{\gamma})$, then $\cL$ is also $(1-\gamma)$-probably $\nu^*$-tolerably  robustly-reliable learner for $\cH$
under marginal $\cD_X$.
\end{theorem}
\begin{proof}
Let $\cD'$ be a distribution with the same marginal $\cD_X$ as $\cD$. Let $S\sim \cD'^m$.  By uniform convergence (\cite{anthony2009neural} Theorem 4.10), for each $h'\in\cH_{\nu^*}(\cD')= \{h\in \cH\mid \err_{\cD'}(h)\le\nu^*\}$, we have that $\err_{S}(h')\le 2\nu^*$ with probability at least $1-\gamma$ over the draw of $S$. For each such $S$, for all $S'\in\cA_\eta(S)$, and for all $x\in\cX$, if $\cL_{S'}(x)=(y,\eta')$ for $\eta'\geq \eta$ then $y= h(x)$ for any $h\in\cH$ with $\err_S(h)\le 2\nu^*$ since $\cL$ is $2\nu^*$-tolerably  robustly-reliable for any sample $S'$. In particular, the prediction agrees with each $h'\in\cH_{\nu^*}(\cD')$. Therefore $\cL$ is a $(1-\gamma)$-probably $\nu^*$-tolerably  robustly-reliable learner for $\cH$
under marginal $\cD_X$.
\end{proof}

We can use the above reduction, together with our $\nu$-tolerably strongly robustly-reliable learner, to give a probably tolerably robustly-reliable learner for any distribution $\cD$ along with guarantees about its distribution-averaged robust reliability region.

\begin{theorem}
Let $\cH$ be the concept space, and $\cD$ be a distribution such that $\min_{h\in\cH}\err_\cD(h)=\nu^*$. If $m\ge \frac{c}{{\nu^*}^2}(d+\ln\frac{1}{\gamma})$, then there is a $(1-\gamma)$-probably $\nu^*$-tolerably  robustly-reliable learner for $\cH$
under marginal $\cD_X$, with $(2\nu^*,\eta)$-robustly-reliable correctness that satisfies
$$\bbE_{S\sim\cD^m}[\text{RobC}^\cL(\cD,2\nu^*,\eta,S)]\ge (1-2\gamma)\Pr_{x\sim\cD_X}[\Agree\left(H_{2\eta+3\nu^*}(\cD)\right)],$$
where $d$ is the VC-dimension of $\cH$, and $c$ is an absolute constant.
\label{thm:distrib-nr}
\end{theorem}
\begin{proof} Let $h^*\in \argmin_{h\in\cH}\err_\cD(h)$.
Let $S\sim \cD^m$ for some $m\ge \frac{c}{{\nu^*}^2}(d+\ln\frac{1}{\gamma})$. By uniform convergence (\cite{anthony2009neural} Theorem 4.10), with probabilility at least $1-\gamma$ over the draw of $S$, we have that $\err_S(h^*)\le\err_\cD(h^*)+\nu^*\le 2\nu^*$.
By Theorem \ref{thm:theta-ub-nr}, there exists a $2\nu^*$-tolerably strongly robustly-reliable learner $\cL$ such that $\CR^\cL(S,2\nu^*,\eta)\supseteq \Agree\left(H_{2\eta+2\nu^*}(S)\right)$. Further, by Theorem \ref{thm:agnostic-distribution-reduction}, $\cL$ is $(1-\gamma)$-probably $\nu^*$-tolerably  robustly-reliable learner for $\cH$
under marginal $\cD_X$.

Finally, let's consider the $\eta$-robustly-reliable correctness of $\cL$. By another application of uniform convergence, with probability at least $1-\gamma$, $\Agree\left(H_{2\eta+2\nu^*}(S)\right)\supseteq \Agree\left(H_{2\eta+3\nu^*}(\cD)\right)$. As shown above, with probability $1-\gamma$, we also have $\CR^\cL(S,2\nu^*,\eta)\supseteq \Agree\left(H_{2\eta+2\nu^*}(S)\right)$. By a union bound over the failure probabilities, with probability at least $1-2\gamma$ over the draw of $S$, we have $\CR^\cL(S,2\nu^*,\eta)\supseteq \Agree\left(H_{2\eta+3\nu^*}(\cD)\right)$. Finally, by the definition of $(\nu,\eta)$-robustly-reliable correctness, we have that $$\bbE_{S\sim\cD^m}[\text{RobC}^\cL(\cD,2\nu^*,\eta,S)]=\bbE_{S\sim\cD^m}[\Pr_{x\sim\cD_X}[\CR^\cL(S,2\nu^*,\eta)]]\ge (1-2\gamma)\Pr_{x\sim\cD_X}[\Agree\left(H_{2\eta+3\nu^*}(\cD)\right)]$$
\end{proof}

Finally we have the following lower bound on the robustly-reliable correctness for any $(1-\gamma)$-probably $\nu$-tolerably robustly-reliable learner. The key idea in establishing this lower bound is the creation of two distributions with the same marginal but nearly-consistent with two hypotheses which are close in error.

\begin{theorem}
Let $\cL$ be a $(1-\gamma)$-probably $\nu$-tolerably robustly-reliable learner for hypothesis class $\cH$ under marginal $\cD_X$. Given a large enough sample size $m=|S|\ge\frac{c}{\epsilon^2}\left(d+\ln\frac{1}{\delta}\right)$, we have
$$\bbE_{S\sim\cD^m}[\text{RobC}^\cL(\cD,\nu+\epsilon,\eta,S)]\le \Pr[\Agree(\cH_{2\eta+\nu-\epsilon}(\cD))]+2\gamma+3\delta,$$
where $c$ is an absolute constant and $\cD$ is the distribution with marginal $\cD_X$ consistent with $h^*$.
\label{thm:probably_robustly-reliable_lower_bound_distributional_agnostic}
\end{theorem}

\begin{proof}
Let  $\cL$ be a $(1-\gamma)$-probably $\nu$-tolerably robustly-reliable learner for $\cH$ under marginal $\cD_X$ and let $h^*\in\cH$ such that $\err_\cD(h^*)\le\nu$.  Let $x\not\in \Agree(\cH_{2\eta+\nu-\epsilon}(\cD))$. To prove the theorem, it suffices to prove that $\Pr_{S \sim \cD^m}[x \in \CR^\cL(S,\nu+\epsilon,\eta) ] \leq 2\gamma+3\delta$.

By definition of the disagreement region, we can select some $h_1,h_2 \in \cH_{2\eta+\nu-\epsilon}(\cD)$ such that $h_1(x)\neq h_2(x)$. Assume WLOG that $h_1(x)\neq h^*(x)$. Let $E^*=\{(x,y)\mid h^*(x)\ne y\}$ and $E_1=\{(x,y)\mid h_1(x)\ne y\}$. Note that $\Pr_\cD[E^*]\le\nu$ and $\Pr_\cD[E_1]\le2\eta+\nu-\epsilon$. Let $\cD'$ be a distribution which modifies $\cD$ by re-assigning a probability mass of $p=\min\{2\eta-\epsilon,\Pr_\cD[E_1\setminus E^*]\}$ from points $(x,y)\in E_1\setminus E^*$ to corresponding points with flipped labels $(x,1-y)$. Formally, using Corollary \ref{cor:redistribution}, there are distributions $\cD_1$ and $\cD_2$ such that $\cD=p\cD_1+(1-p)\cD_2$ with $\texttt{supp}(\cD_1)\subseteq E_1\setminus E^*$, and we define $\cD'=p\cD_1'+(1-p)\cD_2$ where $\cD_1'$ is the distribution corresponding to drawing from $\cD_1$ but with flipped labels. Notice that $\cD'$ has the same marginal $\cD_X$ distribution as $\cD$. Moreover, $\err_{\cD'}(h_1)\le \nu$. Indeed, if $\Pr_\cD[E_1\setminus E^*]\le 2\eta-\epsilon$, we have $\err_{\cD'}(h_1)\le \Pr_{\cD'}[E^*]=\Pr_{\cD}[E^*]\le \nu$. Else, observe that $\err_{\cD'}(h_1)\le \Pr_{\cD'}[E_1] \le \Pr_\cD[E_1]-(2\eta-\epsilon)\le \nu$.

Drawing from $\cD$ is equivalent to first flipping a random biased coin, with probability $p$ drawing according to $\cD_1$ (call these points `red'), and otherwise drawing from $\cD_2$. Let $S \sim \cD^m$ and $S'$ be the dataset obtained by flipping labels of all the `red' points in $S$. Notice that $S'\sim\cD'^m$. We now consider three bad events of total probability at most $2\gamma+3\delta$: (A) $\cL$ is not $(\nu+\epsilon)$-tolerably robustly-reliable for all datasets in $\cA_\eta(S)$, (B) $\cL$ is not $(\nu+\epsilon)$-tolerably robustly-reliable for all datasets in $\cA_\eta(S')$, and (C) $d(S,S') > 2\eta$.  Indeed events (A) and (B) occur with probability at most $\gamma+\delta$ each since $\cL$ is given to be a $(1-\gamma)$-probably $\nu$-tolerably robustly-reliable learner for $\cH$ under $\cD_X$, and by uniform convergence (\cite{anthony2009neural} Theorem 4.10) for each $h$ with $\err_\cD(h)\le \nu$ we have $\err_S(h)\le \nu+\epsilon$ with probability at least $1-\delta$. Finally, a point drawn according to $\cD$ is red with probability $p\le 2\eta-\epsilon$. By an application of the Hoeffding's inequality, there are at most $2\eta$ red points in $S$ with probability at least $1-\delta$, i.e. (C) occurs with probability at most $\delta$.

We claim that if none of these bad events occur, then $x \not\in \CR^\cL(S,\nu+\epsilon,\eta)$.
To prove this, assume for contradiction that none of the bad events occur and $x \in \CR^\cL(S,\nu+\epsilon,\eta)$. Let $\Tilde{S}_{h_1,h^*}$ denote a relabeling of $S$ such that exactly half the red points are labeled according to $h^*$ (the remaining half using $h_1$). Note that since bad event (C) did not occur, this means that $\Tilde{S}_{h_1,h^*} \in \cA_\eta(S)$ and $\Tilde{S}_{h_1,h^*} \in \cA_\eta(S')$,  Now, since $x \in \CR^\cL(S,\nu+\epsilon,\eta)$ and $\Tilde{S}_{h_1,h^*} \in \cA_\eta(S)$, it must be the case that $\cL_{\Tilde{S}_{h_1,h^*}}(x) = (y,\eta')$ for some $\eta'\geq \eta$. But, if $y \neq h^*(x)$ then this implies bad event (A), and if $y \neq h'(x)$ then this implies bad event (B).  So, $x$ cannot be in $\CR^\cL(S,\nu+\epsilon,\eta)$ as desired.
\end{proof}

\section{Discussion}
In this work, for the first time we provide correctness (specifically robust reliability) guarantees for instance-targeted training time attacks, along with
guarantees about the size of the robustly-reliable region.
One implication of our results is a combined robustness against test-time attacks as well, where an adversary makes imperceptible or irrelevant perturbations to the {\em test} examples in order to induce errors. Specifically, even without being given any knowledge of the kinds of perturbations such an adversary could make, our algorithms in the worst case will simply abstain (output $\eta<0$) or produce an $\eta$ less than the perturbation budget of the adversary if the adversary is performing both kinds of attacks together. In the agnostic case, the adversary could cause us to make a mistake, but only if {\em every} low-error hypothesis in the class would have also made a mistake. An interesting open question is to provide usefulness guarantees for test time attacks, for reliable predictors: that is, nontrivial bounds on the probability
that an adversary's perturbation of the test point
can cause the predictor to abstain.

\section*{Acknowledgments}
We thank Hongyang Zhang for useful discussions in the early stages of this work.
 This material is based on work supported by the National Science Foundation under grants  CCF-1910321, IIS-1901403, SES-1919453, and CCF-1815011; an AWS Machine Learning Research Award; an Amazon  Research Award; a Bloomberg Research Grant; a Microsoft Research Faculty Fellowship; and by the Defense Advanced Research Projects Agency under cooperative agreement HR00112020003.
 The views expressed in this work do not necessarily reflect the position or the policy of the Government and no official endorsement should be inferred. Approved for public release; distribution is unlimited.

\bibliographystyle{plainnat}
\bibliography{robustness}

\appendix

\section{Additional lower bounds in the realizable setting}
\label{sec:additional}

We show how we can strengthen the lower bound of Theorem \ref{thm:probably_robustly-reliable_lower_bound_distributional} for learners that satisfy the somewhat stronger condition of being a $(1-\gamma)$-uniformly robustly-reliable learner (defined below).  This condition requires that with probability at least $1-\gamma$ over the draw of the {\em unlabeled} sample $S_X$, for all targets $h^* \in \cH$, the sample $S$ produced by labeling $S_X$ by $h^*$ should be a good sample.  In contrast, Theorem \ref{thm:probably_robustly-reliable_lower_bound_distributional} requires only that for any $h^*$, with probability at least $1-\gamma$ over the draw of $S$, $S$ is a good sample (different samples may be good for different target functions).

\begin{definition}
$\cL$ is a {\bf $(1-\gamma)$-uniformly robustly-reliable learner} for class $\cH$ under marginal distribution $\cD_X$ if with probability at least $1-\gamma$ over the draw of an unlabeled sample $S_X \sim \cD_X^m$, for all $h^* \in \cH$, for all $S' \in \cA_\eta(S)$ (where $S=\{(x,h^*(x)) : x\in S_X\}$), for all $x \in \cX$, if $\cL_{S'}(x) = (y,\eta')$ for $\eta' \geq \eta$ then $y=h^*(x)$.
\end{definition}

\begin{theorem}
Let $\cL$ be a $(1-\gamma)$-uniformly robustly-reliable learner for hypothesis class $\cH$ under marginal $\cD_X$. Then, for any $h^*\in\cH$, for $S \sim \cD^m$ (where $\cD$ is the distribution over examples labeled by $h^*$ with marginal $\cD_X$), with probability $1-\gamma$ over the draw of $S$ we must have
$$\CR^\cL(S,h^*,\eta) \subseteq \Agree\left(\cB_S^\cH(h^*,2\eta)\right).$$
Moreover, if $m\geq \frac{c}{\epsilon^2}(d+\ln\frac{1}{\delta})$, where $c$ is an absolute constant and $d$ is the VC-dimension of $\cH$, then with probability $1-\delta$, $\Agree(\cB_S^\cH(h^*,2\eta)) \subseteq \Agree(\cB_\cD^\cH(h^*,2\eta-\epsilon))$. So, $\bbE_{S\sim\cD^m}[\text{RobC}^\cL(\cD,\eta,S)]\le \Pr[\Agree(\cB_\cD^\cH(h^*,2\eta-\epsilon))]+\gamma+\delta.$
\label{thm:probably_robustly-reliable_lower_bound_sample}
\end{theorem}

\begin{proof}
The proof is similar to the proof for Theorem \ref{thm:strongly_robustly-reliable_lower_bound}.
We are given that $\cL$ is a $(1-\gamma)$-uniformly robustly-reliable learner.  This means that with probability $1-\gamma$, the unlabeled sample $S_X$ has the property that the learner would be robustly-reliable for any target function in $\cH$.  Assume this event occurs, and consider any $x \not\in \Agree(\cB_{S}^\cH(h^*,2\eta))$.  We will show that $x$ cannot be in the $\eta$-robustly-reliable region.  First, since $x \not\in \Agree(\cB_{S}^\cH(h^*,2\eta))$, there must exist some $h' \in \cB_{S}^\cH(h^*,2\eta)$ such that $h'(x) \neq h^*(x)$. Next, let $S_\cA$ denote a labeling of $S_X$ such that exactly half the points in $\Delta_S=\{x\in S_X\mid h'(x)\ne h^*(x)\}$ are labeled according to $h^*$ (the remaining half using $h'$, for convenience assume $|\Delta_S|$ is even). Notice that $S$ satisfies $S_\cA\in\cA_\eta(S)$ since $h'\in \cB_{S}(h^*,2\eta)$. Also for $S'=\{(x_i,h'(x_i)\mid x_i\in S_X\}$ labeled by $h'$, we have $S_\cA\in\cA_\eta(S')$. Now, assume for contradiction that $x \in \CR^\cL(S,h^*,\eta)$. This means that $\cL_{S_\cA}(x)=(y,\eta')$ for some $\eta'\geq \eta$.  However, if $y\ne h^*(x)$, the learner is incorrectly confident for (true) dataset $S$ since $S_\cA\in\cA_\eta(S)$.
Similarly, if $y=h^*(x)$, the learner is incorrectly confidnet for sample $S'$ since $h'(x)\ne h^*(x)$. This contradicts our assumption about $S_X$.

The second part of the theorem statement follows directly from standard uniform convergence bounds.
\end{proof}

\section{Relating $\cA_\eta^{mal}$ to other models}
\label{sec:relation}

In this section we provide some observations that relate our instance-targeted malicious noise adversarial model from Section \ref{sec:malicious} to models from prior and current work. We begin by noting that learnability against the sample-based attacks defined above implies learnability under attacks which directly corrupt the distribution $\cD$ with malicious noise.

\begin{lemma}
Suppose $\cD$ is a distribution over $\cX\times\cY$. Under the malicious noise model of \cite{kearns1993learning}, the learner has access to samples from $\cD_\eta(\cA')=(1-\eta)\cD+\eta\cA'$, where $\cA'$ is an arbitrary adversarial distribution over $\cX\times\cY$. Let $S\sim\cD^m$ and $S'\sim\cD_\eta(\cA')^m$. Then there exists adversary $A_1\in\cA^{\text{mal}}_\eta$ such that $A_1(S)$ is distributed identically as $S'$.
\end{lemma}\label{lem:contains-mal}
\begin{proof}
Select $A_1$ as the adversary which flips a coin for each example in $S$ and with probability $\eta$ replaces the example with one generated by $\cA'$.
\end{proof}

This implies that if we have a learner that is robust to adversaries in $\cA_\eta^{\text{mal}}$, it will also work for the malicious noise model of \cite{kearns1993learning}.
Next, we observe that the $\cA_\eta^{mal}$ model is a special case of the $\cA_{\eta+\epsilon}$ model in the following sense.

\begin{lemma}
Let $S$ be a sample of large enough size $m=|S|\ge \frac{1}{\epsilon^2}\log\frac{2}{\delta}$, and $S'$ be a corruption of $S$ by malicious adversary $\cA^{\text{mal}}_\eta$. Then with probability at least $1-\delta$ over the randomness of $\cA^{\text{mal}}_\eta$, the corrupted sample satisfies $S'\in\cA_{\eta+\epsilon}(S)$.
\end{lemma}

\begin{proof}
By Hoeffding's inequality, with probability at least $1-\delta$, the $\cA^{\text{mal}}_\eta$ adversary may corrupt at most $\eta+\epsilon$ fraction of examples in $S$.
\end{proof}

\section{Further discussion of  \cite{awasthi2017power} and modification to our setting}
\label{appendix:ABL}
\cite{awasthi2017power} define the adversary in the malicious noise model as making its choices in a sequential order as training examples are drawn, with knowledge of the points drawn so far but not of those still to come or of the random choices the algorithm will subsequently make.  In our setting, we want with high probability over the sample $S \sim \cD^m$, the draw of $v \sim \mbox{Bernoulli}(\eta)^m$ and any randomness in the algorithm, for the algorithm to succeed for all $S' \in \cA^{mal}(S,v)$.  This corresponds to an adversary that can make its choices after both (a) the entire sample $S$ is drawn and (b) any internal randomness in the algorithm has been fixed. Difference (a) is easy to handle, and indeed the analysis of \cite{awasthi2017power} goes through directly. To address (b) we will modify the algorithm to make it deterministic; this increases its label complexity, though it remains polynomial.    Here, we review the relevant portions of the algorithm and analysis in \cite{awasthi2017power} and describe needed changes.
We begin with three useful facts about log-concave distributions:

\begin{lemma}[\cite{lovasz2007geometry}]
Assume $\cD$ is an isotropic log-concave distribution over $\R^d$. Then $\Pr_{x \sim \cD}(||x|| \geq \alpha\sqrt{d}) \leq e^{-\alpha+1}$.  Furthermore, if $d=1$ then $\Pr_{x \sim \cD}(x \in [a,b]) \leq |b-a|$.
\label{lem:isotropic_facts}
\end{lemma}

\begin{lemma}[\cite{balcan2013active} Lemma 3] Assume $\cD$ is an isotropic log-concave distribution over $\R^d$. Then there exists $c$ such that for any two
unit vectors $u$ and $v$ in $\R^d$ we have
$c\theta(u, v) \le \text{dis}_{\cD}(u, v)=\text{Pr}_{(x,y)\sim\cD}\text{sgn}(\langle u,x\rangle)\ne\text{sgn}(\langle v,x\rangle)$.
\label{lem:ilc-angle}
\end{lemma}

\begin{theorem}[\cite{balcan2013active} Theorem 14] Let $\beta\ge0$ be a sufficiently small constant. Assume that $\cD$ is an isotropic $\beta$ logconcave distribution in $\R^d$. Then Hanneke's disagreement coefficient \cite{hanneke2007bound} is given by $\theta(\epsilon)=O(d^{\frac{1}{2}+\frac{\beta}{2\ln 2}}\log(1/\epsilon))$.
\label{thm:logconcave-isotropic-disagreement-coefficient}
\end{theorem}

Now we show how we adapt the algorithm and analysis from \cite{awasthi2017power}.
Algorithms \ref{alg:active-algorithm-malicious} and \ref{alg:outlier-removal} below contain the algorithm from \cite{awasthi2017power} for learning from malicious noise, adapted to our model and restated to our notation.   As noted above, one key change to the algorithm is to make it deterministic: specifically, Step 2(b) of Algorithm  \ref{alg:active-algorithm-malicious} examines the labels of all points in set $W$ rather than just a random sample in computing and optimizing weighted hinge-loss.  We verify that the arguments in  \cite{awasthi2017power} continue to apply to this algorithm in the $A^{mal}_\eta$ model (in particular, that the noisy and non-noisy points continue to satisfy the properties used in the analysis), and highlight where particular care or modifications are needed.

\newcommand{\rounds}{s}
\addtocounter{algorithm}{1}
\begin{algorithm}[h]
\caption{\label{alg:active-algorithm-malicious}{\sc Efficient Learning of Linear Separators in the $\cA^{mal}_\eta$ model}
}
\textbf{Input}: allowed error rate $\epsilon$, probability of failure $\delta$,
sequences of sample sizes $n_k > 0$, $k = 1,2,3,...$,
a sequence of cut-off values $b_k >0$,
a sequence of hypothesis space radii $r_k >0$,
a sequence of removal rates $\xi_k$,
a sequence of variance bounds $\sigma^2_k$, precision value $\kappa$; weight vector $w_0$.
Note that compared to \cite{awasthi2017power}, we have replaced the random sampling inside $W$ with just directly using loss with respect to weighting $p$.
\begin{itemize}
\item[1.] Place the first $n_1$ examples in $S'$ into a working set $W$.
\item[2.] {For $k=1,\ldots, \rounds = \lceil \log_2(1/\epsilon) \rceil$}
\begin{itemize}
\item[a.] {Apply Algorithm~\ref{alg:outlier-removal} to $W$ with parameters
$u \leftarrow {w}_{k-1}$, $\gamma \leftarrow b_{k-1}$, $r \leftarrow r_k$, $\xi \leftarrow \xi_k$, $\sigma^2 \leftarrow \sigma^2_k$ and let $q$ be the output function $q: W \rightarrow [0,1]$ . Normalize $q$ to form a probability
   distribution $p$ over $W$.}
\item[b.] {Find ${v}_k \in B({w}_{k-1},r_k)$, $\Vert {v}_k \Vert_2 \leq 1$,
to approximately minimize hinge loss in $W$ weighted by $p$:
$\ell_{\tau_k}({v}_k ,p) \leq \min_{w \in B({w}_{k-1},r_k) \cap B(0,1)}
            \ell_{\tau_k}(w,p) + \kappa/32$.  \- \-\\
 Normalize ${v}_k$ to have unit length, yielding
${w}_k = \frac{{v}_k}{\Vert {v}_k \Vert_2}$.}
\item[c.] Let $W=\{\}$.  Scan $S'$ until
$n_{k+1}$ data points $x$ have been found such that $|{w}_{k}\cdot x| < b_k$, and put them into $W$.  (If $S'$ is exhausted before this step is completed, then output failure.)
\end{itemize}
\end{itemize}
{\textbf{Output}: weight vector $w_{\rounds}$ of error at
most $\epsilon$ with probability $1-\delta$.}
\end{algorithm}

\begin{algorithm}[h]
\caption{\label{alg:outlier-removal}{\sc Localized soft outlier removal procedure} \cite{awasthi2017power}
}

 \textbf{Input}: a set {\reg $W$} of samples;
                the reference unit vector $u$;
                desired radius $r$;
                a parameter $\xi$ specifying the desired bound on the fraction of clean examples removed;
                a variance bound $\sigma^2$
\begin{itemize}
\item[1.] Find $q: {\reg W} \rightarrow [0,1]$ satisfying the following constraints:
\begin{itemize}
\item[a.] for all $x \in {\reg W}$, $0 \leq q(x) \leq 1$
\item[b.] $\frac{1}{\reg |W|} \sum_{(x,y) \in {\reg W}} q(x) \geq 1 - \xi$
\item[c.] for all
$w \in B(u, r) \cap B({\bf 0},1),
\; \frac{1}{\reg |W|} \sum_{x \in {\reg W}} q(x) (w\cdot x)^2   \leq \sigma^2$.
\end{itemize}
\end{itemize}
\textbf{Output}: A function $q: {\reg W} \rightarrow [0,1]$.
\end{algorithm}

The first step in the analysis (Theorem 4.2 in \cite{awasthi2017power}) is to show that with high probability, the convex program in Algorithm \ref{alg:outlier-removal} has a feasible solution $q^*$, specifically one in which $q^*(x)=0$ for each noisy point and $q^*(x)=1$ for each clean point.  This analysis makes the worst-case assumption that every adversarial point is placed inside the band $|{w}_{k}\cdot x| < b_k$.  Moreover, since $q^*(x)=0$ for each noisy point, there is no dependence on its location or label, so this $q^*$ remains feasible in our setting. Also, the analysis uses the fact that the non-noisy points are a true random sample from $\cD$ (which is true in the $\cA^{mal}_\eta$ model but not the $\cA_\eta$ model used in Section \ref{sec:nastynoise}). For completeness, we present the theorem and summarize its analysis.

First, let $\eta'$ denote the value $\eta$ divided by the probability mass under $\cD$ of the band $|{w}_{k-1}\cdot x| < b_{k-1}$.  This is an upper bound on the probability a random example in the band is noisy even if the adversary places all of its points inside the band.  By Chernoff bounds, $n_{k} \geq \mathrm{poly}(1/\eta', 1/\delta)$ suffices so that with probability at least $1 - \delta$, at most a $2 \eta'$ fraction of the points in $W$ are controlled by the adversary.  Assume this is indeed the case.  Also, let $D_{w_{k-1}, b_{k-1}}$ denote the distribution $\cD$ restricted to the band  $|{w}_{k-1}\cdot x| < b_{k-1}$.

\setcounter{equation}{5}
\begin{theorem}[Analog of Theorem 4.2 of \cite{awasthi2017power} for the $\cA^{mal}_\eta$ noise model]
\label{thm:lopchop}
For any $C > 0$, there is a constant $c$ and a polynomial $p$ such that,
for all $\xi > 2 \eta'$ and all $0 < \gamma < C$,
if $n_{k+1} \geq
p(1/\eta', d, 1/\xi, 1/\delta,1/\gamma,1/r)$, then,
with probability $1 - \delta$, the output $q$ of Algorithm~\ref{alg:outlier-removal} satisfies the following:
\begin{itemize}
\item $\sum_{x \in {\reg W}} q(x) \geq (1 - \xi) |{\reg W}| $ (a fraction $1 - \xi$ of
the weight is retained)
\item For all unit length $w$ such that $\Vert w-u \Vert_2 \le r$,
\begin{equation}
\label{e:lopchop.varbound}
\frac{1}{\reg |W|} \sum_{x \in {\reg W}} q(x) (w \cdot x)^2
   \leq c (r^2 + \gamma^2).
\end{equation}
\end{itemize}
Furthermore, the algorithm can be implemented in polynomial time.
\end{theorem}
\begin{proof}[Proof Sketch]
The theorem is proven by analyzing the clean data points, namely those that are not under the adversary's control, using a pseudodimension-based uniform convergence analysis and properties of log-concave distributions. This uses the fact that the {\em non-noisy} points are a true random sample from the underlying distribution, so this does not apply to the $\cA_\eta$ model used in Section \ref{sec:nastynoise}.

Specifically, the proof uses the following two properties of isotropic log-concave distributions.
\begin{lemma}[Lemma 4.3 of \cite{awasthi2017power}]
\label{lemma:all.var}
If we draw $\ell$ times i.i.d. from
$D_{w_{k-1}, b_{k-1}}$ to form $X_C$, with probability $1 - \delta$, we have that for
any unit length $a$,
\[
\frac{1}{\ell} \sum_{x \in X_C} (a \cdot x)^2
   \leq E[(a\cdot x)^2]
       +  \sqrt{\frac{O(d \log(\ell/\delta) (d + \log (1/\delta)))}{\ell}}.
\]
\end{lemma}

\begin{lemma}[Lemma 3.4 of \cite{awasthi2017power}]
\label{lemma:var.small}
Assume that $D$ is isotropic log-concave.
For any $c_3$, there is a constant $c_4$ such that, for
all $0 < \gamma \leq c_3$,
for all $a$ such that $\Vert u-a \Vert_2 \le r$ and $\Vert a \Vert_2 \le 1$
\[
\E_{x \sim D_{u,\gamma}} ((a \cdot x)^2)
  \leq c_4 (r^2 + \gamma^2) .
\]
\end{lemma}
\noindent
The two lemmas above, along with the fact that the non-noisy points in the $\cA^{mal}_\eta$ model are indeed distributed according to $\cD$, imply that for all
$w \in B(u,r)$,
\[
\frac{1}{|W|} \sum_{x \in W} q^*(x) (a \cdot x)^2
  \leq 2 \E[(a \cdot x)^2]
 \leq c (r^2 + \gamma^2),
\]
for an appropriate constant $c$, where  $q^*(x)=0$ for each noisy point and $q^*(x)=1$ for each non-noisy point as defined above.  Finally, an efficient separation oracle for the convex program yields the theorem.
\end{proof}

The next step is a further analysis of the working set $W$ in each round $k$ of Algorithm \ref{alg:active-algorithm-malicious}.  The analysis partitions $W$ into the ``clean'' set $W_C$ drawn from the true distribution $D_{w_{k-1}, b_{k-1}}$ of points from $\cD$ subject to lying inside the previous band, and the ``dirty'' set $W_D$ controlled by the adversary.  Lemma 4.5 in \cite{awasthi2017power} upper-bounds the size of $W_D$, and its analysis makes the worst-case assumption that every adversarial point lies inside the band.  Since our setting is identical in terms of the probabilistic selection of which points are controlled by the adversary (differing only in allowing the adversary to make its choice after all of $S$ is drawn), this argument also goes through directly. Specifically, Lemma 4.5, adapted to our setting, is as follows:

\begin{lemma}[Analog of Lemma 4.5 of \cite{awasthi2017power} for the $\cA^{mal}_\eta$ noise model]
\label{lemma:few.noisy}
There is an absolute positive constant $c$ such that, with
probability $1 - \frac{\delta}{6 (k+k^2)}$,
\begin{equation}
\label{e:few.noisy}
|W_D| \leq c \eta n_k {M}^k \leq \frac{ c M \eta n_k}{\epsilon}
\end{equation}
\end{lemma}
\begin{proof}[Proof Sketch]
Since the underlying distribution of non-noisy points is isotropic log-concave, the
probability that a random example from $\cD$ falls in the band is
$\Omega(M^{-k})$.  By Chernoff bounds, this implies that $O({n_k {M}^k})$ examples from $\cD$ are sufficient so that with probability at least $(1 - \frac{\delta}{12 (k+k^2)})$,
$n_k$ examples fall inside the band. In the $\cA^{mal}_\eta$ model, the probability that each example drawn
is noisy is $\eta$, and in the worst case, the adversary can place each noisy example inside the band. Since with probability at least $(1 - \frac{\delta}{12 (k+k^2)})$ the total number of examples drawn is $O({n_k {M}^k})$, Chernoff bounds imply that for some constant $c$, with
probability at least $1 - \frac{\delta}{12 (k+k^2)}$ we have
$|W_D| \leq c \eta n_k M^k$.  The second inequality follows from the fact that $k \leq \lceil \log_M (1/\epsilon) \rceil$.
\end{proof}

The bound of Lemma \ref{lemma:few.noisy} above is then used in Lemma 4.7 of \cite{awasthi2017power}, which relates the average hinge loss of examples in $W$ when weighted according to the solution $q$ found by Algorithm \ref{alg:outlier-removal} (scaled to sum to 1 over $W$) to the average hinge loss of the clean examples $W_C$ weighted uniformly. In particular, it shows that the soft outlier-removal performed by Algorithm \ref{alg:outlier-removal} is comparable to correctly identifying and removing the adversarial data with respect to the loss of any $w \in B({w}_{k-1},r_k)$. Lemma \ref{lemma:few.noisy} is used in equation (11) of the proof of Lemma 4.7 (see below), and is the only property of $W_D$ used in the proof.

Specifically, let $p$ denote the weighting $q$ over $W$ produced by Algorithm \ref{alg:outlier-removal}, scaled to sum to 1.  Constraints a and b in the definition of $q$ ensure that the total variation distance between $p$ and the uniform distribution over $W$ is at most $\xi$.  Now, let $\ell(w,p)$ be the average hinge-loss of points in $W$ weighted by $p$, and let $\ell(w,W_C)$ denote the average hinge loss of the clean examples $W_C$ weighted uniformly.  Lemma 4.7 of \cite{awasthi2017power} relates these as follows:

\setcounter{equation}{7}
\begin{lemma}[Analog of Lemma 4.7 of \cite{awasthi2017power} for the $\cA^{mal}_\eta$ noise model]
\label{lemma:clean.vs.filtered.lc}
There are absolute constants $C_1$, $C_2$ and $C_3$ such that,
with probability $1 - \frac{\delta}{2 (k+k^2)}$,
if we define $z_k = \sqrt{r_k^2 + {b^2_{k-1}}}$, then
for any $w \in B({w}_{k-1},r_k)$, we have
\begin{equation}
\label{e:clean.by.filtered}
\ell(w, W_C)
\leq \ell(w,p)
        + \frac{C_1 \eta}{\epsilon}
    \left(1 + \frac{z_k}{\tau_k}
                        \right) + \kappa/32
\end{equation}
and
\begin{equation}
\label{e:filtered.by.clean}
\ell(w,p)
  \leq 2 \ell(w,W_C)
        + \kappa/32
        +  \frac{C_2 \eta}{\epsilon}
        + C_3 \sqrt{\frac{\eta}{\epsilon}} \times \frac{z_k}{\tau_k}.
\end{equation}
\end{lemma}
\begin{proof}[Proof Sketch]
Given the lemmas shown above, the proof now follows exactly as in \cite{awasthi2017power}.  In particular, the key properties of the data used are the size of the set $W_D$, which comes from Lemma \ref{lemma:few.noisy} analyzed above and is used inequation (\ref{e:few.noisy.lc}) below, and the fact that the clean data is distributed according to $\cD$ (which is true in the $\cA^{mal}_\eta$ model).  Specifically,
from the lemmas above, with probability at least {
$1 - \frac{\delta}{2(k+k^2)}$, there are constants
$K_1$, $K_2$ and $K_3$ such that  
\begin{eqnarray}
\label{e:tvar.q}
\frac{1}{|W|}
  \sum_{x \in W} q(x) (w \cdot x)^2 & \leq & K_1 z^2_k \\
\label{e:few.noisy.lc}
|W_D| & \leq & \frac{K_2 \eta n_k}{\epsilon} \\
\label{e:clean.var}
\frac{1}{|W_C|} \sum_{(x,y) \in W_C} (w \cdot x)^2
     & \leq & K_3 z_k^2.
\end{eqnarray}
}
Assume these indeed hold. Equation (\ref{e:tvar.q}) and the fact that
{$\sum_{x \in W} q(x) \geq (1 - \xi_k) |W| \geq |W|/2$ imply
$
\sum_{x \in W} p(x) (w \cdot x)^2 \leq {2 K_1}z^2_k.
$
}
Combining this with equation (\ref{e:few.noisy.lc}) and that fact that the total variation distance between $p$ and the uniform distribution over $W$ is at most $\xi$, and using Cauchy-Schwartz, yields the following bound on the weighted loss of the noisy examples in $W$:
{
\begin{equation}
 \label{e:noisy.loss}
  \sum_{(x,y) \in W_D} p(x) \ell(w,x,y) \le
    K_2 \eta/\epsilon  + \xi_k
      + \sqrt{2 K_1 K_2 \eta/\epsilon + \xi_k} \left(\frac{z_k}{\tau_k}\right).
 \end{equation}
}
A similar argument gives the following bound on the weighted loss over all the examples in $W$:
{
\begin{equation}
\label{e:total.loss}
\sum_{(x,y) \in W} p(x) \ell(w,x,y) \le 1
   + \sqrt{2 K_1} \left( \frac{ z_k}{\tau_k} \right).
\end{equation}
}
Next, using (\ref{e:clean.var}) together with Cauchy-Schwartz and the fact that $\sum_{x \in W} q(x) \geq (1 - \xi_k) |W|$ and $q(x)\in [0,1]$, we can upper-bound the average loss over the clean examples:
{
\begin{align*}
\ell(w,W_C)
 & \leq \frac{1}{|W_C|} \left( \sum_{(x,y) \in W} q(x) \ell(w,x,y) \right)
  + 2 \xi_k + \sqrt{ 2 \xi_k K_3 } \left( \frac{z_k}{\tau_k} \right). \\
\end{align*}
}
Finally, $\xi_k$ is chosen small enough so the last two terms above are at most $\kappa/32$.  Applying further manipulation together with (\ref{e:few.noisy.lc}) and (\ref{e:total.loss}) yields the first inequality in the Lemma.  The second inequality follows from partitioning $\ell(w,p)$ into a sum over clean points and a sum over noisy points, and then applying (\ref{e:noisy.loss}) together with (\ref{e:few.noisy.lc}) and the fact that $p(x)\leq 1$ for all $x$.
\end{proof}

These results now can be combined to give the algorithm's guarantee for the $\cA^{mal}_\eta$ noise model.
\begin{theorem}[Analog of Theorem 4.1 in \cite{awasthi2017power} for the $\cA^{mal}_\eta$ noise model]
\label{t:malicious.detailed}
Let distribution $\cD$ over $R^d$ be isotropic log-concave.  Let
$w^*$ be the (unit length) target weight vector.  There are settings
of the parameters of Algorithm~\ref{alg:active-algorithm-malicious},
and positive constants $M$, $C$ and $\epsilon_0$, such that in the $\cA^{mal}_\eta$ noise model, for all
$\epsilon < \epsilon_0$, for any $\delta>0$, if $\eta < C \epsilon$,
$n_k = \mathrm{poly}(d,M^k, \log(1/\delta))$,
and
$\theta(w_0,w^*) < \pi/2$, then after $s=O(\log(1/\epsilon))$ iterations, the algorithm
finds $w_{\rounds}$ satisfying $\err(w_{\rounds}) \leq \epsilon$ with
probability $\geq 1-\delta$.
\end{theorem}
\begin{proof}[Proof Sketch]
First, by a uniform convergence analysis, with high probability for all $w \in B(w_{k-1},r_k)$ we have
$|\err_{D_{w_{k-1}, b_{k-1}}}({w}) - \ell(w,W_C)| \leq \kappa/16$. Assume this holds, as well as equations (\ref{e:clean.by.filtered}) and (\ref{e:filtered.by.clean}) from Lemma \ref{lemma:clean.vs.filtered.lc}. 
Next, we analyze as in \cite{awasthi2017power} but with one key change: Algorithm \ref{alg:active-algorithm-malicious} deterministically computes $\ell(v_k,p)$ in step 2(b) rather than using a random labeled sample from $p$ as in \cite{awasthi2017power}; thus, we do not need to relate the sampled loss and true expected loss with respect to $p$. This change in the algorithm was crucial because in the $\cA^{mal}_\eta$ model, the adversary can effectively choose its corruptions {\em after} observing any internal randomness of the algorithm.
Specifically, we now have:
\begin{eqnarray*}
\err_{D_{w_{k-1}, b_{k-1}}}({w_k}) & \leq & \ell({v_k},W_C) + \kappa/16\\
& \leq & \ell({v_k},p)
   + \frac{C_1 \eta}{\epsilon} \left(1 + \frac{z_k}{\tau_k}
                        \right)
   + \kappa/8
    \\
  & \leq & \ell(w^*,p)
   + \frac{C_1 \eta}{\epsilon} \left(1 + \frac{z_k}{\tau_k}
                        \right)
   + \kappa/8
    \mbox{\hspace{0.1in}(since $w^* \in B({w}_{k-1},r_k)$)} \\
   & \leq & 2 \ell(w^*,W_C)
        +  \frac{C_2 \eta}{\epsilon}
        + C_3 \sqrt{\frac{\eta}{\epsilon}} \times \frac{ z_k}{\tau_k}
     + \frac{C_1 \eta}{\epsilon} \left(1 + \frac{z_k}
                                           {\tau_k} \right)
   + \kappa/4 \\
  & \leq & \kappa/3
        +  \frac{C_2 \eta}{\epsilon}
        + C_3 \sqrt{\frac{\eta}{\epsilon}} \times \frac{ z_k}{\tau_k}
   + \frac{C_1 \eta}{\epsilon} \left(1 + \frac{z_k}{\tau_k} \right)
   + \kappa/2.
\end{eqnarray*}
Using the fact that $z_k = O(\tau_k)$,
an $\Omega(\epsilon)$ bound on
$\eta$ is sufficient so
that $\err_{D_{w_{k-1}, b_{k-1}}}({w_k}) \le \kappa$ with probability
$(1-\frac{\delta}{k+k^2})$. 

Finally, the theorem follows from an inductive argument showing that after $k$ iterations, with high probability we have $\err_\cD({w}_k) \leq M^{-k}$.  In particular, since $\err_\cD({w}_{k-1}) \leq M^{-(k-1)}$ by induction, the angle between $w_{k-1}$ and $w^*$ is $O(M^{k-1})$. Also, $r_k$ is chosen to be proportional to $M^{-(k-1)}$ so the angle between $v_k$ and $w_{k-1}$ is also $O(M^{k-1})$.  $M$ is chosen a sufficiently large constant so that by properties of isotropic log-concave distributions (Theorem 4 of \cite{balcan2013active}), $w_k$ has error at most $M^{-k}/4$ with respect to $w_{k-1}$ outside the band $|{w}_{k-1}\cdot
x| \leq b_{k-1}$, and $w_{k-1}$ has error at most $M^{-k}/4$ with respect to $w^*$ outside the band. So, $w_k$ has error at most $M^{-k}/2$ with respect to $w^*$ outside the band.  From above, we have that $\err_{D_{w_{k-1}, b_{k-1}}}({w_k}) \le \kappa$. Combining this with a bound on the probability mass in the band gives error at most $M^{-k}/2$ inside the band.  So, $\err_\cD({w}_k) \leq M^{-k}$ as desired, proving the theorem.
\end{proof}

\section{Supporting results for Section \ref{sec:non-realizable}}\label{app:agnostic}

\begin{lemma}\label{lem:redistribution}
Let $(\Omega,\cF,P)$ be a probability space. Let $A\in\cF$ and $\alpha\le P(A)$. Then we have probability spaces $(A,2^A\cap \cF,P_1)$ and $(\Omega,\cF,P_2)$ such that for any $S\in\cF$, $P(S)=\alpha P_1(A\cap S)+(1-\alpha)P_2(S)$.
\end{lemma}
\begin{proof}
If $P(A)=0$, we have $\alpha=0$ and the result holds for $P_2=P$ and any $P_1$. Also, if $\alpha=1$, we have $P(A)=1$ (or, $P(\Omega\setminus A)=0$) and $P_1(T)=P(T)$ for any $T\in 2^A\cap \cF$ defines a probability space over $A$. Further, for any $S\in\cF$, $P_1(A\cap S)=P(A\cap S)=P(S)-P((\Omega\setminus A)\cap S)=P(S)$ and the desired result holds for any $P_2$.

Assume $P(A)>0$ and $\alpha<1$. Define $P_1(T)=P(T)/P(A)$ for any $T\in 2^A\cap \cF$. Notice $(A,2^A\cap \cF,P_1)$ is a probability space. In particular, $P_1(A)=1$ and $0\le P_1(T)\le 1$ for any $T\in 2^A\cap \cF$. For any $S\in\cF$, define
$$P_2(S)=\frac{1}{1-\alpha}P(S)-\frac{\alpha}{1-\alpha}P_1(S\cap A).$$
We have $P_2(\Omega)=\frac{1}{1-\alpha}P(\Omega)-\frac{\alpha}{1-\alpha}\frac{P(\Omega\cap A)}{P(A)}=1$. Also for any $S\in\cF$,
\begin{align*}
    P_2(S)&=\frac{1}{1-\alpha}P(S)-\frac{\alpha}{1-\alpha}\frac{P(S\cap A)}{P(A)}\\
    &\ge \frac{1}{1-\alpha}P(S)-\frac{1}{1-\alpha}P(S\cap A)\ge 0,
\end{align*}
since $P(A)\ge\alpha$ and $P(S)\ge P(S\cap A)$. Also,
\begin{align*}
    P_2(S)&=\frac{1}{1-\alpha}\left(P(S)-\alpha\frac{P(S\cap A)}{P(A)}\right)\\
    &= \frac{1}{1-\alpha}\left(1-P(\Omega\setminus S)-\alpha\frac{P(A)-P(A\setminus (S\cap A))}{P(A)}\right)\\
    &=1-\frac{1}{1-\alpha}\left(P(\Omega\setminus S)-\alpha\frac{P(A\setminus S)}{P(A)}\right)\\
    &\le 1-\frac{1}{1-\alpha}\left(P(\Omega\setminus S)-P(A\setminus S)\right)\le 1.
\end{align*}
$\sigma$-additivity of $P_2$ follows from that of $P$ and $P_1$. Thus, $(\Omega,\cF,P_2)$ is a probability space and the desired relation holds.
\end{proof}

We have the following corollary to Lemma \ref{lem:redistribution}.

\begin{corollary}\label{cor:redistribution}
Let $\cD$ be a distribution over set $S$. Let $\alpha\le\Pr_\cD[A]$ for some $A\subseteq S$. Then there exist distributions $\cD_1$ and $\cD_2$, with $\texttt{supp}(\cD_1)\subseteq A$ and $\texttt{supp}(\cD_2)\subseteq S$, such that $\cD=\alpha \cD_1+(1-\alpha)\cD_2$.
\end{corollary}

\end{document}